\documentclass[sigconf,nonacm]{acmart}
\usepackage{amsmath}
\usepackage{amsthm}

\usepackage{caption}
\usepackage{subcaption}
\usepackage{multirow}
\usepackage{rotating}
\usepackage{graphicx}
\usepackage{booktabs}
\usepackage{enumitem}

\newtheorem{lemma}{Lemma}
\newtheorem{proposition}{Proposition}
\newtheorem{theorem}{Theorem}

\AtBeginDocument{%
  }

\begin{document}

\title{RankFormer: Rethinking Time Series Transformers through Representation Capacity Analysis and Progressive Decoding}

\author{Peiwang Tang}
\affiliation{%
  \institution{University of Science and Technology of China}
  \city{Hefei}
  \state{Anhui}
  \country{China}
}
\affiliation{%
  \institution{State Key Laboratory of Cognitive Intelligence}
  \city{Hefei}
  \state{Anhui}
  \country{China}
}
\affiliation{%
  \institution{Shanghai Black Wing Asset Management}
  \city{Shanghai}
  \country{China}
}
\email{tpw@mail.ustc.edu.cn}

\author{Cheng HE}
\affiliation{%
  \institution{University of Science and Technology of China}
  \city{Hefei}
  \state{Anhui}
  \country{China}
}
\affiliation{%
  \institution{State Key Laboratory of Cognitive Intelligence}
  \city{Hefei}
  \state{Anhui}
  \country{China}
}
\affiliation{%
  \institution{Shanghai Black Wing Asset Management}
  \city{Shanghai}
  \country{China}
}
\email{cheng.he@mail.ustc.edu.cn}

\author{Xijie Liang}
\affiliation{%
  \institution{Shanghai Black Wing Asset Management}
  \city{Shanghai}
  \country{China}
}
\email{lxjie@blackwingasset.com}

\author{Defu Lian}
\affiliation{%
  \institution{University of Science and Technology of China}
  \city{Hefei}
  \state{Anhui}
  \country{China}
}
\affiliation{%
  \institution{State Key Laboratory of Cognitive Intelligence}
  \city{Hefei}
  \state{Anhui}
  \country{China}
}
\authornote{Corresponding author.}
\email{liandefu@ustc.edu.cn}

\renewcommand{\shortauthors}{Tang et al.}

\makeatletter
\renewcommand{\@mkauthors@iii}{%
  \gdef\@currentauthors{}%
  \gdef\@currentaffiliation{}%
  \global\setbox\mktitle@bx=\vbox{%
    \unvbox\mktitle@bx\par\medskip
    \begingroup
    \centering\normalfont
    {\fontsize{11}{14}\selectfont
      Peiwang Tang\textsuperscript{1,2,3}\quad
      Cheng HE\textsuperscript{1,2,3}\quad
      Xijie Liang\textsuperscript{3}\quad
      Defu Lian\textsuperscript{1,2,\ding{41}}\par}
    \smallskip
    {\fontsize{9}{12}\selectfont
      \textsuperscript{1}University of Science and Technology of China, Hefei, Anhui, China\par
      \textsuperscript{2}State Key Laboratory of Cognitive Intelligence, Hefei, Anhui, China\par
      \textsuperscript{3}Shanghai Black Wing Asset Management, Shanghai, China\par}
    \smallskip
    {\fontsize{8}{11}\selectfont
      \href{mailto:tpw@mail.ustc.edu.cn}{\texttt{tpw@mail.ustc.edu.cn}}\quad
      \href{mailto:cheng.he@mail.ustc.edu.cn}{\texttt{cheng.he@mail.ustc.edu.cn}}\par
      \href{mailto:lxjie@blackwingasset.com}{\texttt{lxjie@blackwingasset.com}}\quad
      \textsuperscript{\ding{41}}\href{mailto:liandefu@ustc.edu.cn}{\texttt{liandefu@ustc.edu.cn}}\par}
    \endgroup
    \bigskip
  }%
}
\makeatother

\begin{abstract}
Multivariate time series forecasting plays a critical role in Web and
cyber-physical systems, where large-scale services and IoT infrastructures
generate complex temporal data. Transformer-based models have achieved strong empirical performance in this domain, yet the principles governing the expressive
capacity of their learned representations remain insufficiently understood.
Moreover, many existing models rely on simplified decoding strategies,
including channel-independent modeling and direct projection heads, which
limit their ability to capture complex multivariate dynamics. 
To address these challenges, we propose \textbf{RankFormer}, a novel
encoder-decoder architecture for multivariate time series forecasting.
We establish a rank-based analysis framework to investigate how
tokenization, embedding, and attention designs affect the representation
capacity of time series Transformers. Our analysis reveals that
patchification increases the potential representation capacity by
aggregating multiple temporal observations into each token, while subsequent
encoding operations determine how much of this capacity is realized. Guided
by these insights, we develop a rank-enhancing encoder with Hourglass
Attention to improve rank stability during information interaction.
Furthermore, we introduce the Interactive Progressive Decoder
(IPD), which improves future trajectory modeling through global-guided
cross-channel interaction and progressive refinement. 
Extensive experiments
on real-world long-term forecasting benchmarks demonstrate that RankFormer
achieves competitive performance against state-of-the-art methods.
Code repository: \url{https://github.com/TangPeiwang/Rankformer}.
\end{abstract}



\maketitle

\section{Introduction}

Multivariate time series forecasting is a core problem in modern Web and cyber-physical systems. 
The growing integration of online services and IoT devices produces large volumes of continuously evolving sensor data \cite{zhang2026semixer, zhang2026re, zhang2026time}, supporting applications such as server load balancing, traffic prediction, and resource planning \cite{wang2026eeo, zhang2026enhancing, zhang2026amortized}, as well as broader domains including finance, power systems, and transportation \cite{zhou2021informer, wu2021autoformer, fang2025efficient}. 
Transformer-based models have achieved substantial progress in time series forecasting owing to their ability to capture long-range dependencies \cite{wutimesnet, liu2022pyraformer, zhou2022fedformer}. 
Recent architectural advances, including patch-based tokenization and specialized attention mechanisms, have further improved forecasting performance across diverse benchmarks \cite{liu2024timer, liu2025timer}. 
However, despite this empirical success, the underlying principles that determine the representation capacity of time series Transformers and their ability to translate learned representations into accurate forecasts remain insufficiently understood.

\begin{figure*}[ht]
    \centering
    \begin{subfigure}{0.48\textwidth}
        \centering
        \includegraphics[width=\linewidth]{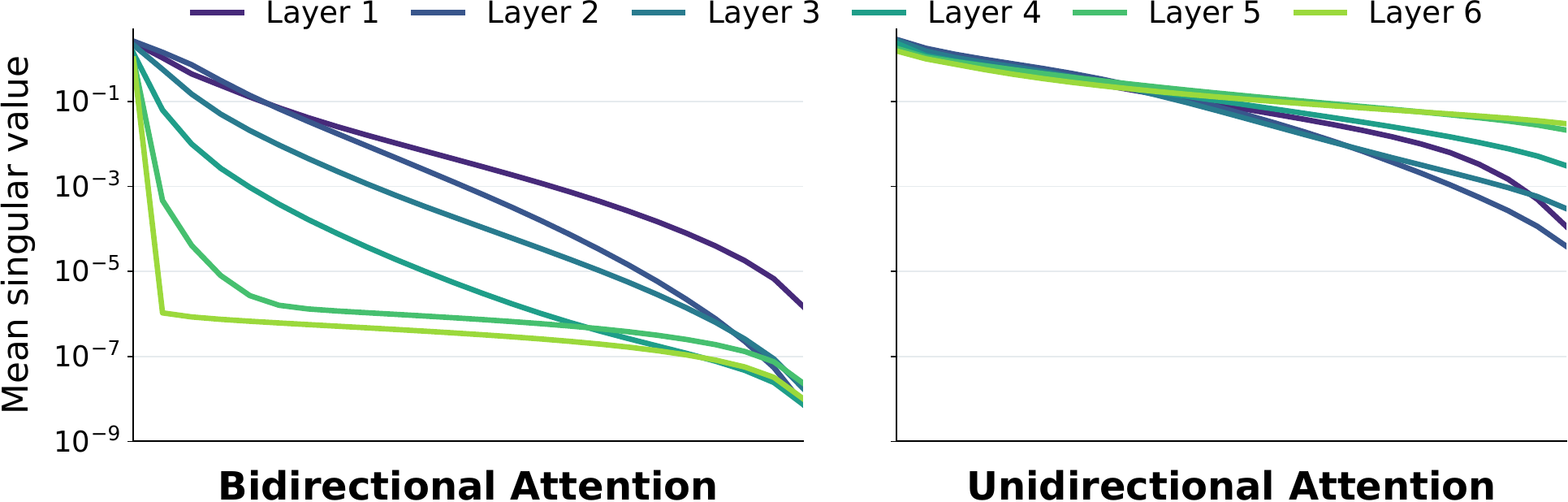}
        \caption{} 
        \label{fig_2a}
    \end{subfigure}
    \hfill
    \begin{subfigure}{0.48\textwidth}
        \centering
        \includegraphics[width=\linewidth]{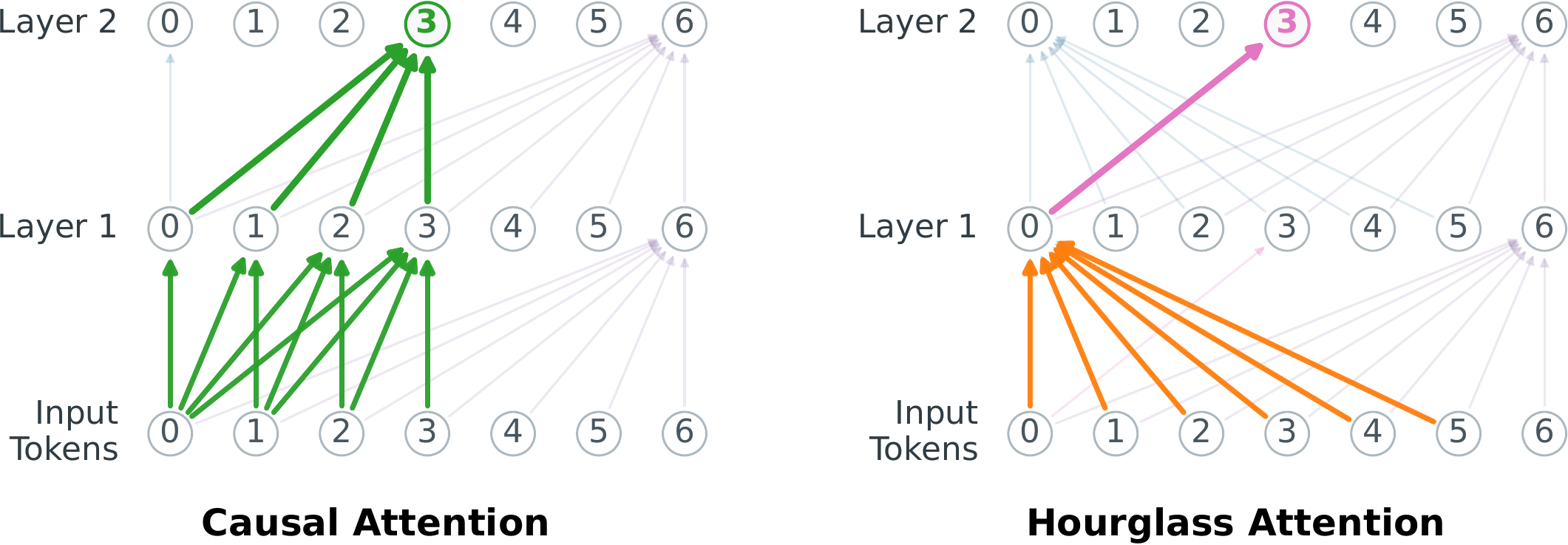}
        \caption{} 
        \label{fig_2b}
    \end{subfigure}
    
    \caption{Empirical validation of rank stability and Hourglass Attention (HgA). (a) Rapid singular value decay shows bidirectional attention induces deep-layer rank collapse, while causal masking preserves representational diversity. (b) Causal masking restricts early tokens (starvation). In contrast, HgA uses a symmetric prefix as a global hub, enabling efficient context retrieval while maintaining rank stability.}
    \label{fig_2}
\end{figure*}

This limited understanding raises two fundamental questions: 

\textbf{First, how do encoder architectures create and preserve expressive representations? }
Existing time series Transformers improve representation capability through
various encoding designs, including tokenization, embedding, and attention
mechanisms, yet how these components jointly shape the capacity of learned temporal representations remains insufficiently understood.
Patchification provides a representative example: by aggregating multiple
temporal observations into each token, it expands the attainable
representation space compared with point-wise tokenization.
However, the realized representation capacity is further shaped by
subsequent encoding operations.
Linear embeddings introduce rank constraints, while bidirectional attention (BiA)
may cause excessive token mixing and representation homogenization, leading
to rank degeneration \cite{noci2022signal, saadamind, dong2021attention,yu2025understanding}.
Although unidirectional attention provides stronger rank stability, its
restricted receptive field limits global information interaction.
Consequently, existing attention designs struggle to simultaneously
maintain representation diversity and achieve effective cross-token
information exchange, limiting the ability of time series Transformers to
fully exploit their representation space.

\textbf{Second, how can forecasting decoders fully exploit the learned representations? }
Although modern encoders can produce increasingly expressive latent features, many existing forecasting architectures still rely on simplified decoding strategies. Channel-independent modeling limits cross-variable interaction, while one-shot projection heads such as Flatten-Linear map encoded representations to the entire prediction horizon through a fixed transformation. 
Such designs overlook the evolving temporal organization and provide limited capability for progressively refining complex future trajectories, which is particularly restrictive for multivariate forecasting with strong cross-channel dependencies.

To address these challenges, we propose \textbf{RankFormer}, a principled
encoder-decoder architecture for multivariate time series forecasting.
For the encoder, we establish a rank-based analysis framework to analyze how
patchification, embedding, and attention topology shape the effective capacity of learned temporal representations.
Based on these observations, we introduce Hourglass Attention (HgA) to
improve rank stability during attention-based interaction, preserving
diverse representations while enabling effective global information
exchange.
For the decoder, we develop an Interactive Progressive Decoder (IPD), which
enhances future trajectory generation through global-context-guided
cross-channel interaction and progressive refinement. 
Together, RankFormer provides a principled approach to improving
representation learning and future trajectory modeling in time series
Transformers.

Our main contributions are summarized as follows:
\begin{itemize}
\item We provide a rank-based analysis of time series Transformer encoders,
revealing how patchification, embedding, and attention operations affect
representation capacity and offering a theoretical explanation for the
effectiveness of patch-based modeling.

\item We propose \textbf{RankFormer}, consisting of HgA and IPD. 
HgA preserves expressive representations through improved attention design, while IPD captures cross-channel dependencies and progressively refines future trajectories.

\item Extensive experiments on widely used long-term forecasting benchmarks validate the theoretical observations and demonstrate consistently strong performance against competitive state-of-the-art methods.

\end{itemize}

\begin{figure*}[ht]
	\centering
	\includegraphics[scale=0.66]{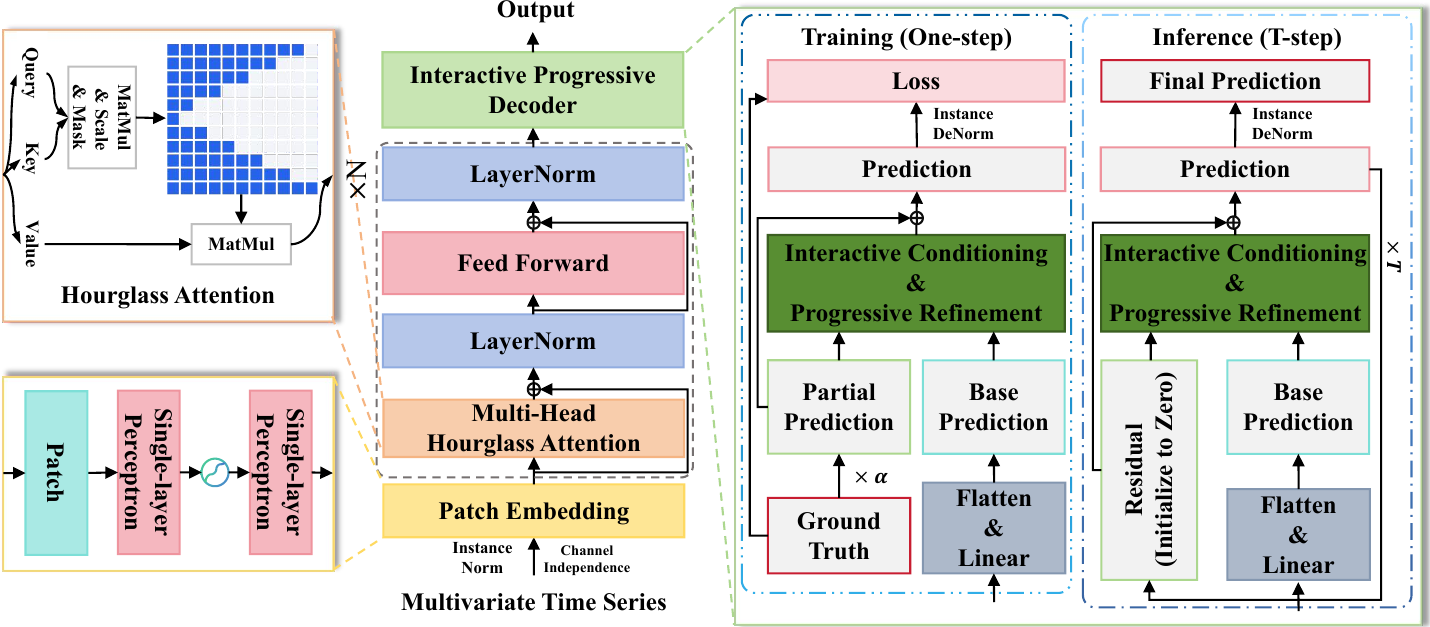}
  \caption{Overall structure of RankFormer. The model preprocesses multivariate time series via Instance Norm and CI, utilizing Patch Embedding and a Multi-Head Hourglass Attention encoder to generate a base prediction. An Interactive Progressive Decoder then refines these features. It trains via one-step residual learning using randomized partial ground truth, while inference employs a $T$-step progressive refinement from a zero-residual state, concluding with Instance Denorm.}
	\label{fig_rankformer}
\end{figure*}

\section{Rank Theory Analysis and Key Theorems}
\label{sec:rank_theory}

In this section, we establish a rank-based theoretical framework to analyze
the core components of time series Transformers. 
We consider the forecasting problem as learning a mapping from historical
observations $\mathbf{X}\in\mathbb{R}^{L\times C}$ to future sequences
$\mathbf{Y}\in\mathbb{R}^{\tau\times C}$.
The model first transforms the input sequence into hidden representations
$\mathbf{H}\in\mathbb{R}^{N\times D}$, whose structure determines how much
independent information can be preserved for subsequent forecasting.
We use matrix rank, $\operatorname{rank}(\mathbf{H})$, as a quantitative
measure to characterize the independent information capacity of hidden
representations.
A higher rank indicates more independent representation directions and
greater potential to preserve diverse information, whereas rank degradation
suggests increased redundancy and reduced representation diversity.

\subsection{Patching as a Rank Expansion Mechanism}

When a univariate sequence $\mathbf{x}\in\mathbb{R}^{L}$ is represented
through point-wise tokenization, the resulting token matrix has a single
feature dimension, and any linear embedding inherits a rank upper bound of
one: 
\[
\operatorname{rank}(\mathbf{x}^{\top}\mathbf{W})
\leq \operatorname{rank}(\mathbf{x})
\leq 1.
\]

\begin{proposition}[Rank Expansion via Patching]
\label{prop:patching_rank}
Partitioning $\mathbf{x}$ into patches of length $P$ and arranging them into
a patch matrix $\mathbf{X}_p\in\mathbb{R}^{N\times P}$, where
$N=\lfloor L/P\rfloor$, gives the attainable rank upper bound:
\[
\operatorname{rank}(\mathbf{X}_p)\leq \min(N,P).
\]
\end{proposition}

When $P>1$, patching expands the maximum attainable rank bound from one to
$\min(N,P)$, increasing the potential number of independent directions
available for subsequent representation learning. This provides a rank-based
perspective for understanding the effectiveness of patch-based modeling in
time series Transformers.

\subsection{Embedding Layer: Breaking Linear Rank Constraints}

To map the patched representation $\mathbf{X}_p$ into a hidden representation
$\mathbf{H}\in\mathbb{R}^{N\times D}$, standard linear projections
($\mathbf{H}=\mathbf{X}_p\mathbf{W}$) preserve the rank constraint of the
input representation:
\[
\operatorname{rank}(\mathbf{H})
\leq
\min(\operatorname{rank}(\mathbf{X}_p),D)
\leq
\min(N,P,D).
\]
When $D>P$, linear embeddings remain constrained by the patch-induced rank
bound, limiting the number of independent directions that can be represented
in the hidden space.

To overcome this linear rank constraint, we consider nonlinear MLP
embeddings:
\[
\mathbf{H}
=
\phi(\mathbf{X}_p\mathbf{W}_1+\mathbf{b}_1)\mathbf{W}_2+\mathbf{b}_2 .
\]

\begin{theorem}[Maximum Rank Embedding]
\label{thm:max_rank_embedding}
Let $\phi$ be a continuous, non-polynomial activation function. For any
$N$ distinct input patches $\mathbf{X}_p\in\mathbb{R}^{N\times P}$, there
exist MLP parameters such that the resulting hidden representation achieves the
maximum attainable rank:
\[
\operatorname{rank}(\mathbf{H})=\min(N,D).
\]
\end{theorem}

This theorem shows that nonlinear embeddings overcome the additional rank
restriction imposed by linear projections, allowing hidden representations
to achieve their maximum attainable rank determined by the sample and hidden
dimensions.

\begin{figure*}[ht]
	\centering
	\includegraphics[scale=0.7]{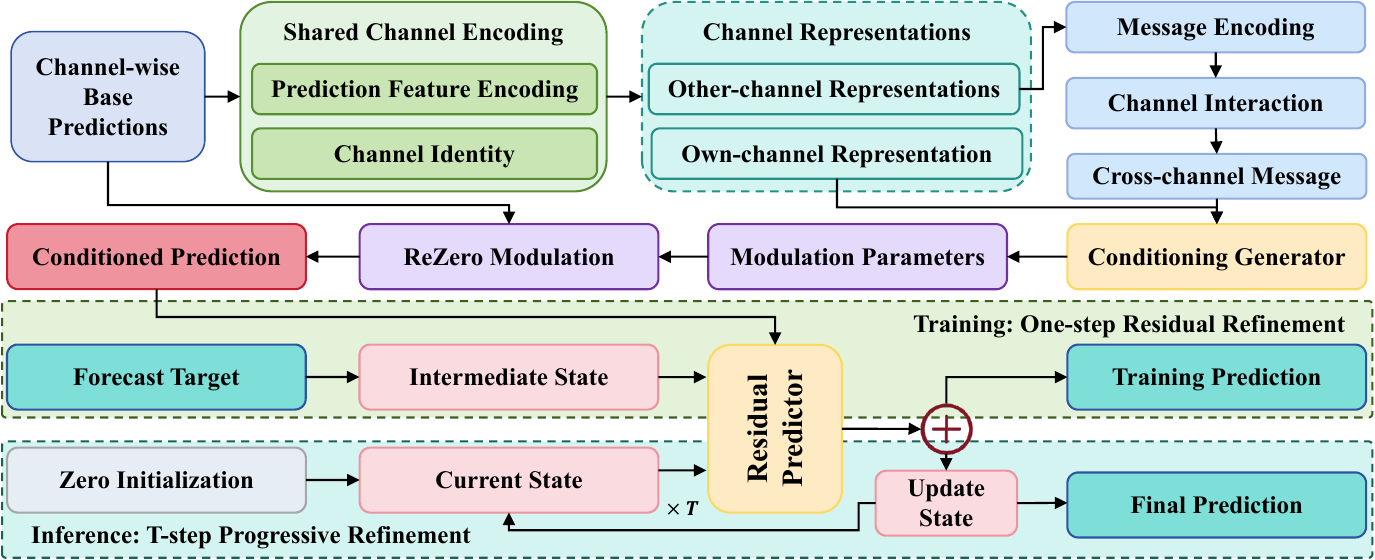}
  \caption{Internal architecture of the Interactive Progressive Decoder.}
	\label{fig_IRD}
\end{figure*}

\subsection{Attention Mechanisms: Rank Preservation via Attention Topology}

BiA enables global information aggregation across tokens,
but its dense mixing structure does not provide an explicit guarantee for
preserving representation rank. Repeated dense token mixing may lead to
representation homogenization and rank degradation in deep networks
\cite{noci2022signal, saadamind}.

To analyze how attention topology affects rank stability, we compare
BiA with unidirectional (e.g., causal) attention.
Let $\mathbf{A}_{\textup{bid}}$ and $\mathbf{A}_{\textup{causal}}$ denote
their attention weight matrices, and let
$\mathbf{O}_{\textup{bid}}$ and $\mathbf{O}_{\textup{causal}}$ denote the
corresponding outputs after value projection and output projection.

\begin{theorem}[Rank Preservation of Causal Attention]
\label{thm:causal_rank_advantage}

Given identical projection parameters
($\mathbf{W}_Q,\mathbf{W}_K,\mathbf{W}_V,\mathbf{W}^{O}$),
causal attention preserves the rank of the value representation and
guarantees:
\[
\operatorname{rank}(\mathbf{O}_{\textup{causal}})
\geq
\operatorname{rank}(\mathbf{O}_{\textup{bid}}).
\]

\end{theorem}

Theorem~\ref{thm:causal_rank_advantage} establishes that causal masking
provides a structural rank-preservation property during attention
aggregation. 
In contrast, BiA may reduce output rank when its learned
aggregation fails to preserve independent information directions in the
value representation (as shown in Figure~\ref{fig_2}(a)).
This analysis highlights the
importance of attention topology in maintaining diverse representations
during token interaction. Rather than favoring a specific topology, it
motivates attention designs that preserve rank stability while enabling
effective information exchange, leading to our proposed Hourglass Attention.

\section{RankFormer}

Guided by the rank-based theoretical framework established above, we propose
\textbf{RankFormer}, as illustrated in Figure~\ref{fig_rankformer}. 
The architecture consists of a rank-guided encoder with HgA, which improves rank stability during attention-based interaction, and an IPD, which captures complex future trajectories through global-context-guided cross-channel interaction and progressive refinement.

\subsection{Patch Embedding}

To mitigate the adverse effects of magnitude variations across different variables, we first apply standard instance normalization to the input sequence $\mathbf{X}$ along the temporal dimension. The normalized time series is then partitioned into patches of length $P$, transforming the input into $\mathbf{X}_{\text{patches}} \in \mathbb{R}^{C \times N \times P}$, where $N$ is the number of patches. As elucidated in Proposition~\ref{prop:patching_rank}, this patching operation structurally expands the rank capacity of the input representation.

Subsequently, to break the inherent rank bound of linear projections, we map these patches into the hidden space $\mathbb{R}^D$ using a two-layer MLP with a ReLU activation \cite{nair2010rectified}:
\[
\mathbf{H}_{\text{embed}} = \text{ReLU}(\mathbf{X}_{\text{patches}} \mathbf{W}_1 + \mathbf{b}_1) \mathbf{W}_2 + \mathbf{b}_2
\]
where $\mathbf{W}_1 \in \mathbb{R}^{P \times D_{\text{mid}}}$ and $\mathbf{W}_2 \in \mathbb{R}^{D_{\text{mid}} \times D}$. As mathematically guaranteed by Theorem~\ref{thm:max_rank_embedding}, this nonlinear formulation ensures that the embedding achieves the theoretical maximum rank of $\min(N, D)$, fully utilizing the representational potential of the high-dimensional hidden space.

\subsection{Hourglass Attention}

BiA enables global information aggregation across tokens,
but repeated dense mixing may cause representation homogenization and rank
degradation in deep networks. In contrast, causal attention provides a
structural rank-preservation guarantee through its triangular topology, but
its directional information flow limits efficient context exchange across
tokens. To combine these complementary properties, we propose HgA, a topology-aware attention mechanism that reorganizes
causal connectivity to maintain rank stability while enabling more effective
information interaction.

HgA introduces an hourglass-shaped attention topology through three design
principles. First, early tokens are assigned broader receptive fields to
aggregate global context from the observed sequence. Second, intermediate
tokens use restricted receptive fields to reduce excessive token mixing and
maintain representation diversity. Third, receptive fields gradually expand
toward later tokens, preserving detailed recent information for downstream
forecasting. Together, these designs redistribute information access within
the sequence while retaining the rank-stability property inherited from
causal attention.

HgA is constructed by a topology-preserving reordering of causal attention.
For an input sequence of length $N$, we define $\rho(i)$ as the maximum
visible key index of query position $i$ under this reordering:
\begin{equation}
\rho(i)=
\begin{cases}
\max(N-2-2i,0), & i\leq\lfloor (N-1)/2\rfloor,\\
2i-N+1, & i>\lfloor (N-1)/2\rfloor .
\end{cases}
\end{equation}

This mapping creates an hourglass-shaped receptive field: the visible
range contracts toward the sequence center and expands toward the sequence
end. The corresponding mask is:
\begin{equation}
\label{eq:hg_mask}
\mathbf M^{\mathrm{Hg}}[i,j]
=
\begin{cases}
0,&j\leq\rho(i),\\
-\infty,&\text{otherwise}.
\end{cases}
\end{equation}

The attention operation is:
\begin{equation}
\mathbf A^{\mathrm{Hg}}
=
\operatorname{Softmax}
\left(
\frac{\mathbf Q\mathbf K^\top}{\sqrt{d_k}}
+
\mathbf M^{\mathrm{Hg}}
\right),
\end{equation}

\begin{equation}
\operatorname{HgAttn}(\mathbf Q,\mathbf K,\mathbf V)
=
\mathbf A^{\mathrm{Hg}}\mathbf V .
\end{equation}

Since HgA is obtained through a reordering of causal connectivity, its
attention matrix retains the rank-preservation property of causal attention
while providing a different receptive field allocation.

\subsection{Interactive Progressive Decoder}
\label{sec:ipd}

The IPD (Figure~\ref{fig_IRD}) addresses the mapping bottleneck
by transforming channel-independent encoder representations into
multivariate forecasts. It first generates a channel-wise base prediction,
then introduces cross-channel conditioning to incorporate interactions among
variables, and finally refines the prediction through progressive residual
updates.

\paragraph{Base Prediction.}

IPD first obtains a channel-independent base prediction from the flattened
encoder representations using a shared linear projection. We denote the
initial prediction of channel $C$ as
$\mathbf{y}^{\mathrm{base}}_C\in\mathbb{R}^{\tau}$, where $\tau$ is the forecast
horizon. This base prediction provides the initialization for subsequent
cross-channel conditioning and progressive residual refinement.

\begin{table*}[t]
\centering
\caption{Results of the Long-Term Time Series Forecasting.  All baseline models were configured consistently with SEER. The reported results represent the average performance over four distinct forecasting horizons: \{96, 192, 336, 720\}.}
\label{tab1}
\renewcommand{\arraystretch}{1}
\resizebox{0.99\linewidth}{!}{%
\begin{tabular}{c|cc|cc|cc|cc|cc|cc|cc|cc|cc|cc}
\toprule
\multicolumn{1}{c|}{\multirow{2}{*}{Models}} & \multicolumn{2}{c|}{RankFormer} & \multicolumn{2}{c|}{SEER} & \multicolumn{2}{c|}{SRSNet} & \multicolumn{2}{c|}{DUET} & \multicolumn{2}{c|}{xPatch} & \multicolumn{2}{c|}{Amplifier} & \multicolumn{2}{c|}{FredFormer} & \multicolumn{2}{c|}{iTransformer} & \multicolumn{2}{c|}{PatchTST} & \multicolumn{2}{c}{DLinear} \\
\multicolumn{1}{c|}{} & \multicolumn{2}{c|}{(ours)} & \multicolumn{2}{c|}{(2026)} & \multicolumn{2}{c|}{(2025)} & \multicolumn{2}{c|}{(2025)} & \multicolumn{2}{c|}{(2025)} & \multicolumn{2}{c|}{(2025)} & \multicolumn{2}{c|}{(2024)} & \multicolumn{2}{c|}{(2024)} & \multicolumn{2}{c|}{(2023)} & \multicolumn{2}{c}{(2023)} \\
\midrule
\multicolumn{1}{c|}{Metrics} & MSE & MAE & MSE & MAE & MSE & MAE & MSE & MAE & MSE & MAE & MSE & MAE & MSE & MAE & MSE & MAE & MSE & MAE & MSE & MAE \\
\midrule
ETTh1 & \textbf{0.410} & \textbf{0.421} & 0.423 & 0.424 & 0.442 & 0.433 & 0.443 & 0.436 & 0.444 & 0.429 & 0.449 & 0.437 & 0.448 & 0.435 & 0.454 & 0.448 & 0.469 & 0.455 & 0.461 & 0.458 \\
\midrule
ETTh2  & \textbf{0.354} & \textbf{0.381} & 0.356 & 0.384 & 0.376 & 0.402 & 0.372 & 0.398 & 0.369 & 0.392 & 0.389 & 0.411 & 0.378 & 0.403 & 0.383 & 0.407 & 0.387 & 0.407 & 0.563 & 0.519 \\
\midrule
ETTm1  & \textbf{0.367} & \textbf{0.377} & 0.377 & 0.381 & 0.385 & 0.395 & 0.390 & 0.393 & 0.387 & 0.383 & 0.383 & 0.396 & 0.385 & 0.398 & 0.407 & 0.410 & 0.387 & 0.400 & 0.404 & 0.408 \\
\midrule
ETTm2  & \textbf{0.268} & \textbf{0.311} & 0.276 & 0.317 & 0.284 & 0.329 & 0.280 & 0.324 & 0.280 & 0.319 & 0.280 & 0.326 & 0.281 & 0.324 & 0.288 & 0.332 & 0.281 & 0.326 & 0.354 & 0.402 \\
\midrule
Weather & \textbf{0.234} & \textbf{0.256} & 0.246 & 0.265 & 0.250 & 0.277 & 0.251 & 0.273 & 0.247 & 0.266 & 0.247 & 0.275 & 0.256 & 0.278 & 0.258 & 0.278 & 0.259 & 0.281 & 0.265 & 0.315 \\
\midrule
Solar & \textbf{0.212} & \textbf{0.217} & 0.221 & 0.224 & 0.250 & 0.281 & 0.237 & 0.233 & 0.277 & 0.268 & 0.222 & 0.256 & 0.227 & 0.267 & 0.233 & 0.262 & 0.270 & 0.307 & 0.330 & 0.401 \\
\midrule
Electricity & \textbf{0.162} & \textbf{0.246} & 0.169 & 0.251 & 0.189 & 0.276 & 0.172 & 0.259 & 0.184 & 0.266 & 0.174 & 0.268 & 0.178 & 0.270 & 0.178 & 0.270 & 0.216 & 0.304 & 0.225 & 0.319 \\
\midrule
Traffic & 0.431 & 0.249 & \textbf{0.422} & \textbf{0.245} & 0.494 & 0.307 & 0.451 & 0.269 & 0.500 & 0.283 & 0.485 & 0.315 & 0.434 & 0.289 & 0.428 & 0.282 & 0.555 & 0.362 & 0.625 & 0.383 \\
\midrule
$\mathrm{1^{st} Count}$ & 7 & 7 & 1 & 1 & 0 & 0 & 0 & 0 & 0 & 0 & 0 & 0 & 0 & 0 & 0 & 0 & 0 & 0 & 0 & 0 \\
\bottomrule
\end{tabular}}
\end{table*}

\paragraph{Cross-channel Conditioning.}

To introduce multivariate interactions while preserving channel-specific
information, IPD constructs an interaction-aware conditioning representation
for each channel. The base prediction is first projected into a feature space:
\begin{equation}
\mathbf{u}_C
=
\operatorname{GELU}
(\mathbf{P}\mathbf{y}^{\mathrm{base}}_C+\mathbf{b}_P).
\end{equation}

A learnable channel embedding $\mathbf{e}_C$ is combined with the prediction
feature:
\begin{equation}
\mathbf{f}_C=[\mathbf{u}_C;\mathbf{e}_C],
\end{equation}
where $\mathbf{f}_C$ represents the channel-specific conditioning feature.

The cross-channel branch transforms each channel representation into a
message feature:
\begin{equation}
\mathbf{m}_j
=
\mathbf{W}_m\mathbf{f}_j+\mathbf{b}_m .
\end{equation}

For channel $C$, the received information from other variables is aggregated
as:
\begin{equation}
\mathbf{m}^{\mathrm{cross}}_C
=
\sum_{j\neq C}
\mathbf{M}_{Cj}\mathbf{m}_j,
\end{equation}
where $\mathbf{M}$ denotes the learnable channel interaction matrix.

The own-channel feature and cross-channel message jointly generate
horizon-wise modulation parameters:
\begin{equation}
\begin{bmatrix}
\boldsymbol{\gamma}_C\\
\boldsymbol{\beta}_C
\end{bmatrix}
=
\mathbf{G}\mathbf{f}_C
+
\lambda\mathbf{K}\mathbf{m}^{\mathrm{cross}}_C
+
\mathbf{b}_G .
\end{equation}

The conditioned prediction is obtained through ReZero modulation:
\begin{equation}
\mathbf{y}^{\mathrm{cond}}_C
=
\mathbf{y}^{\mathrm{base}}_C
+
\rho_c
(
\mathbf{y}^{\mathrm{base}}_C
\odot
\boldsymbol{\gamma}_C
+
\boldsymbol{\beta}_C
),
\end{equation}
where $\rho_c$ is initialized to zero.

\paragraph{Progressive Residual Refinement.}

Instead of directly mapping encoder representations to the future sequence
through a static projection, IPD models forecasting as residual refinement
from an intermediate prediction state.

During training, a completion ratio
$r\sim\mathcal{U}(0,1)$ is sampled to construct:
\begin{equation}
\mathbf{y}^{\mathrm{mid}}_C
=
r\mathbf{y}_C
+
\sigma\boldsymbol{\epsilon}_C ,
\qquad
\boldsymbol{\epsilon}_C\sim\mathcal{N}(0,\mathbf I).
\end{equation}
where $\mathbf{y}_C$ is the normalized ground-truth sequence and
$\sigma$ controls the noise magnitude.

The residual correction is predicted as:
\begin{equation}
\Delta\mathbf y_C
=
\mathbf W_{\mathrm{state}}\mathbf y_C^{\mathrm{mid}}
+
\mathbf W_{\mathrm{cond}}\mathbf y_C^{\mathrm{cond}}
+
\mathbf b .
\end{equation}

The final training prediction is:
\begin{equation}
\hat{\mathbf y}_C
=
\mathbf y_C^{\mathrm{mid}}
+
(1-r)\Delta\mathbf y_C .
\end{equation}

\paragraph{Offline-folded Inference.}

During inference, IPD initializes the prediction state as
$\mathbf{y}_C^{(0)}=\mathbf{0}$ and performs $T$ residual updates:
\begin{equation}
\mathbf{y}_C^{(t)}
=
\mathbf{y}_C^{(t-1)}
+
\frac{1}{T}
\left(
\mathbf{W}_{\mathrm{state}}\mathbf{y}_C^{(t-1)}
+
\mathbf{W}_{\mathrm{cond}}\mathbf{y}_C^{\mathrm{cond}}
+
\mathbf{b}
\right),
\quad t=1,\ldots,T.
\end{equation}

Since the residual update is affine in the prediction state and
$\mathbf{y}_C^{\mathrm{cond}}$ remains fixed across iterations,
the $T$ updates can be analytically combined.
Specifically, defining
\begin{equation}
\mathbf{S}_T
=
\frac{1}{T}
\sum_{k=0}^{T-1}
\left(
\mathbf{I}
+
\frac{\mathbf{W}_{\mathrm{state}}}{T}
\right)^k,
\end{equation}
the final prediction becomes
\begin{equation}
\mathbf{y}_C^{(T)}
=
\mathbf{B}_{\mathrm{eff}}
\mathbf{y}_C^{\mathrm{cond}}
+
\mathbf{b}_{\mathrm{eff}},
\end{equation}
where
$\mathbf{B}_{\mathrm{eff}}=\mathbf{S}_T\mathbf{W}_{\mathrm{cond}}$
and $\mathbf{b}_{\mathrm{eff}}=\mathbf{S}_T\mathbf{b}$.
These equivalent parameters are precomputed after training,
allowing IPD to reproduce the progressive refinement through a
single affine mapping during inference.

\section{Experiments}
We evaluate RankFormer on diverse real-world long-term time series
forecasting benchmarks spanning Web- and IoT-relevant applications, demonstrating its superior representational capacity and stable progressive decoding against state-of-the-art (SOTA) baselines.


\paragraph{Datasets.} 
We evaluate our model on nine widely adopted real-world benchmarks:
Electricity, Traffic, Weather, Exchange-Rate, Solar, and four ETT datasets
(ETTh1, ETTh2, ETTm1, ETTm2).
These datasets cover diverse applications relevant to Web and IoT systems,
including transportation, energy management, and environmental monitoring,
alongside financial forecasting.

\paragraph{Baselines.} 
We benchmark RankFormer against 9 recent SOTA forecasting models. To ensure a comprehensive evaluation, these baselines encompass advanced Transformer-based architectures (SEER \cite{qiu2026seer}, iTransformer \cite{liuitransformer}, PatchTST \cite{nietime}, Fredformer \cite{piao2024fredformer}), linear-centric networks (DLinear \cite{zeng2023transformers}), and other highly competitive contemporary approaches (SRSNet \cite{wu2025enhancing}, DUET \cite{qiu2025duet}, xPatch \cite{stitsyuk2025xpatch}, Amplifier \cite{fei2025amplifier}).

\paragraph{Experimental Setup.} 
We evaluate all models under standard long-term forecasting settings, $\tau\in\{96,192,336,720\}$.
Performance is measured
by  Mean Squared Error (MSE) and Mean Absolute Error (MAE). Baseline results are taken from original publications or reproduced using official implementations under the same settings.

\subsection{Main Results}

As shown in Table~\ref{tab1}, we evaluate the long-term forecasting performance of RankFormer against SOTA baselines. All models use a fixed input sequence length of 96, with results averaged across four forecasting horizons. RankFormer demonstrates comprehensive superiority, achieving the best MSE and MAE on 7 out of 8 benchmark datasets (securing 14 out of 16 top metrics). While the recent SEER model shows competitive performance and achieves the best results on the Traffic dataset, RankFormer maintains a strong second place with marginal differences. Compared with established baselines such as PatchTST and iTransformer,
RankFormer achieves lower forecasting errors across diverse applications
relevant to Web and IoT systems, including energy management (Solar)
and environmental monitoring (Weather).
These results demonstrate the strong and consistent forecasting
performance of RankFormer across diverse real-world scenarios.

\subsection{Representation Capacity Analysis}

\begin{figure}[h]
    \centering
    \includegraphics[width=\linewidth]{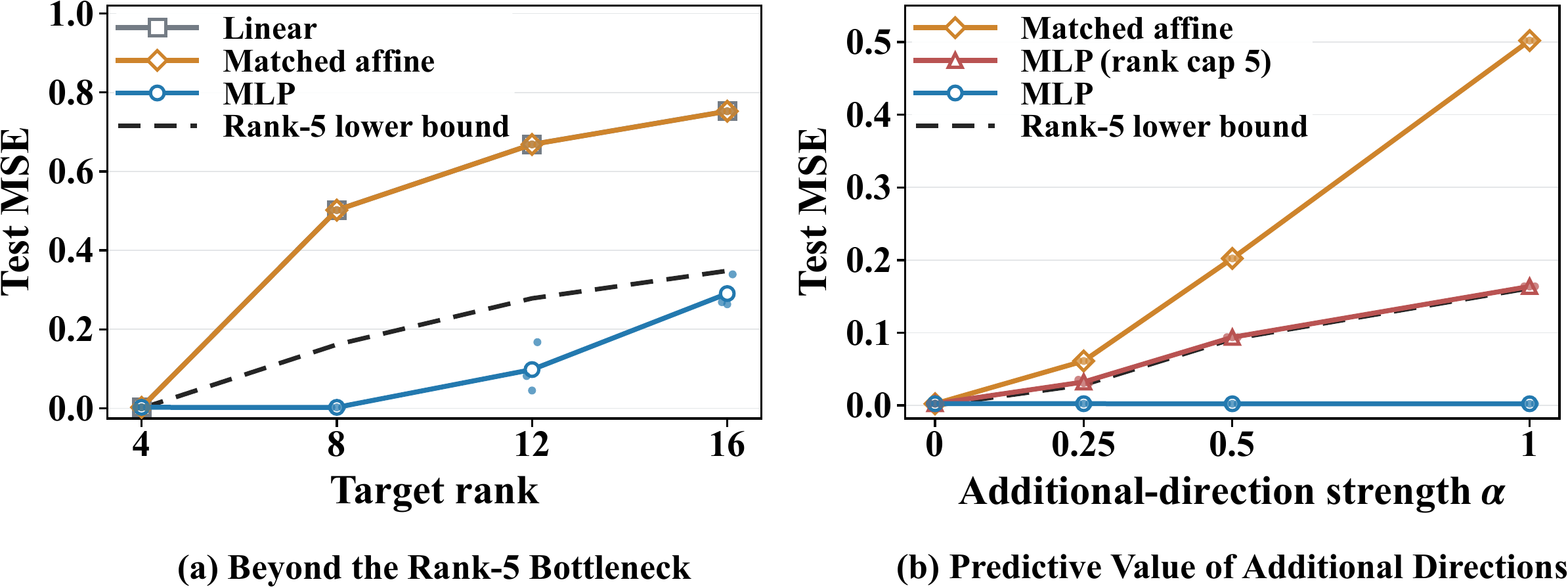}
    \caption{Predictive benefits of additional embedding rank capacity. Lines show three-seed means, small dots show individual seeds, and dashed curves show the MSE lower bound for rank-at-most-five predictions.}
    \label{fig:embedding_rank_capacity}
\end{figure}

\subsubsection{Predictive Value of Embedding Rank Expansion}
\label{sec:embedding_rank_capacity}

We construct synthetic targets with controlled rank to evaluate the
predictive benefit of nonlinear embeddings. We compare linear,
parameter-matched affine, and MLP embeddings using identical attention
and prediction heads that preserve the embedding rank constraints.
Both affine embeddings have a rank upper bound of five.
Figure~\ref{fig:embedding_rank_capacity}(a) shows comparable performance
at target rank 4. At higher target ranks, the MLP achieves MSE below
the rank-five lower bound, while affine embeddings remain constrained.
In (b), increasing the contribution of higher-rank target components
substantially degrades the rank-constrained MLP, whereas the unconstrained
MLP maintains stable performance.
These results demonstrate that nonlinear embeddings can exploit higher
attainable rank to improve prediction when target structures exceed
the linear rank constraint.

\subsubsection{Impact of Representation Rank Constraints}
\label{sec:embedding_rank_constraints}

To investigate the impact of representation rank constraints on forecasting,
we impose rank limits of 1, 4, and 8 on position-enriched patch
representations in a standard Transformer encoder. The constraints are
maintained during training and inference, with all other settings unchanged.
Table~\ref{tab:embedding_pe_rank_constraints} shows that rank-one and
rank-four constraints consistently increase MSE and MAE across datasets,
whereas the relaxed rank-eight constraint largely restores performance
toward the unconstrained baseline. Notably, this degradation persists
despite the presence of FFNs and residual connections.
These results demonstrate that restrictive representation rank constraints
can impair forecasting performance, supporting the importance of
preserving sufficient independent representation directions.

\begin{table}[t]
    \centering
    \caption{Forecasting performance under different rank constraints on
position-enriched embeddings. The best results are in bold.}
    \label{tab:embedding_pe_rank_constraints}
    \small
    \resizebox{\linewidth}{!}{%
    \begin{tabular}{c|cc|cc|cc|cc}
        \toprule
        \multirow{2}{*}{Dataset}
        & \multicolumn{2}{c|}{Unconstrained}
        & \multicolumn{2}{c|}{Rank $\leq 1$}
        & \multicolumn{2}{c|}{Rank $\leq 4$}
        & \multicolumn{2}{c}{Rank $\leq 8$} \\
        & MSE & MAE & MSE & MAE & MSE & MAE & MSE & MAE \\
        \midrule
        ETTh1
        & \textbf{0.503} & \textbf{0.469}
        & 0.549 & 0.511
        & 0.509 & 0.473
        & 0.504 & 0.472 \\
        ETTh2
        & \textbf{0.380} & 0.414
        & 0.385 & 0.417
        & 0.398 & 0.421
        & 0.383 & \textbf{0.414} \\
        ETTm1
        & \textbf{0.369} & \textbf{0.390}
        & 0.385 & 0.402
        & 0.372 & 0.393
        & 0.372 & 0.392 \\
        ETTm2
        & 0.271 & 0.326
        & 0.273 & 0.327
        & 0.274 & 0.327
        & \textbf{0.270} & \textbf{0.324} \\
        Weather
        & \textbf{0.228} & \textbf{0.263}
        & 0.231 & 0.266
        & 0.231 & 0.266
        & 0.229 & 0.265 \\
        Exchange
        & \textbf{0.391} & \textbf{0.423}
        & 0.430 & 0.443
        & 0.416 & 0.437
        & 0.397 & 0.425 \\
        \bottomrule
    \end{tabular}%
    }
\end{table}

\begin{figure}[t]
    \centering
    \includegraphics[width=\linewidth]{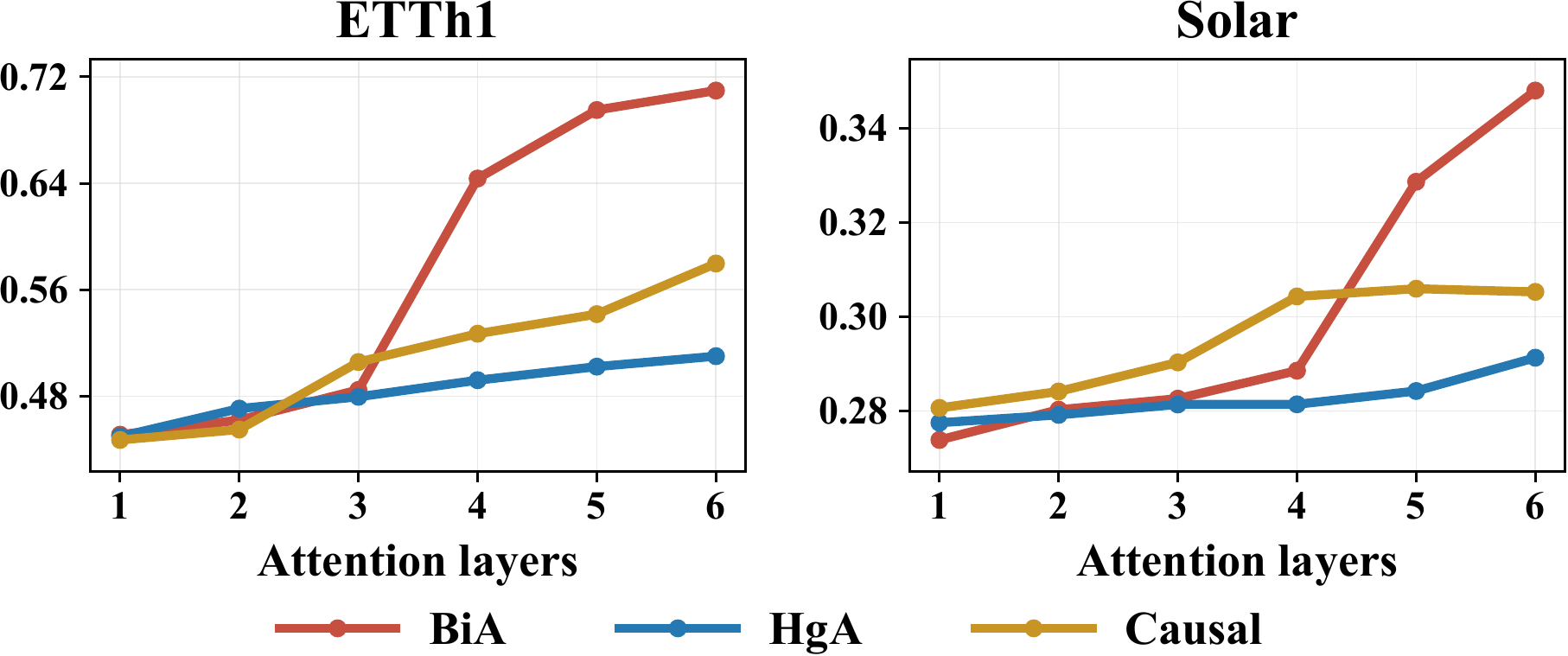}
    \caption{Effect of Network Depth on Forecasting with Different Pure Attention Mechanisms.}
    \label{fig:hga1_depth_mse_ETTh1_Solar}
\end{figure}

\begin{table}[t]
    \centering
    \caption{Performance comparison of six-layer pure-attention models.
    The best results are in bold.}
    \label{tab:pure_attention_bia_hga}
    \small
    \resizebox{\linewidth}{!}{%
    \begin{tabular}{ccccccc}
        \toprule
        Method & Metric & ETTh1 & ETTh2 & ETTm1
        & Weather & Exchange \\
        \midrule
        \multirow{2}{*}{BiA}
        & MSE & 0.710 & 0.418 & 0.648 & 0.279 & 0.375 \\ 
        & MAE & 0.583 & 0.433 & 0.530 & 0.303 & 0.414 \\
        \midrule
        \multirow{2}{*}{HgA}
        & MSE & \textbf{0.510} & \textbf{0.407} & \textbf{0.421}
        & \textbf{0.273} & \textbf{0.371} \\
        & MAE & \textbf{0.480} & \textbf{0.428} & \textbf{0.422}
        & \textbf{0.293} & \textbf{0.412} \\
        \bottomrule
    \end{tabular}
    }
\end{table}

\subsubsection{Robustness to Increasing Attention Depth}

To evaluate the impact of attention depth, we compare BiA, causal attention,
and HgA in pure-attention networks with one to six stacked attention layers.
All variants share the same architecture except for the attention mechanism.
Figure~\ref{fig:hga1_depth_mse_ETTh1_Solar} shows that BiA suffers from
substantial performance degradation as attention depth increases, whereas
HgA maintains more stable forecasting performance. Causal attention also
improves depth robustness, but its performance varies across datasets.
Table~\ref{tab:pure_attention_bia_hga} further demonstrates that HgA
achieves consistent improvements over BiA across multiple datasets,
especially in deeper attention stacks.
These results indicate that HgA improves robustness under increasingly deep
attention interactions by preserving more stable representations. This
empirical behavior is consistent with the rank-preservation analysis,
showing that attention topology plays an important role in maintaining
representation diversity during deep information exchange.

\subsubsection{Effect of Training-Time Attention Rank Constraints}

To investigate how attention rank constraints affect forecasting performance,
we impose rank constraints $\operatorname{rank}(\mathbf{A})\leq r$ with
$r\in\{1,3\}$ on every attention head of a six-layer HgA model. The
constrained models are trained under the same settings as the unconstrained
baseline.
Table~\ref{tab:hga_training_rank_constraints} shows that restricting
attention rank consistently degrades forecasting performance across
datasets. Severe rank constraints ($r=1$) lead to substantial performance
drops, while a relaxed constraint ($r=3$) reduces the degradation. 
These results demonstrate that overly restrictive attention rank constraints
can impair the effectiveness of HgA, supporting the role of maintaining
sufficient representation diversity during attention interaction. 

\begin{table}[t]
    \centering
    \caption{Training-Time Attention Rank Constraints in Six-Layer Pure HgA.
    For the rank-constrained columns, values are percentage changes
    relative to the Unconstrained model.}
    \label{tab:hga_training_rank_constraints}
    \small
    \renewcommand{\arraystretch}{1}
    \resizebox{\linewidth}{!}{%
    \begin{tabular}{c|cc|cc|cc}
        \toprule
        \multirow{2}{*}{Dataset} & \multicolumn{2}{c|}{Unconstrained}
        & \multicolumn{2}{c|}{$\mathrm{rank}(A)\leq 1$}
        & \multicolumn{2}{c}{$\mathrm{rank}(A)\leq 3$} \\
        & MSE & MAE & MSE & MAE & MSE & MAE \\
        \midrule
        ETTh1
        & 0.510 & 0.481
        & +12.44 & +7.36
        & +0.23 & +0.04 \\
        ETTh2
        & 0.406 & 0.427
        & +0.93 & -0.55
        & +1.88 & +0.72 \\
        ETTm1
        & 0.422 & 0.421
        & +37.72 & +19.04
        & +7.27 & +3.89 \\
        ETTm2
        & 0.293 & 0.338
        & +6.92 & +4.07
        & +3.60 & +2.02 \\
        Weather
        & 0.273 & 0.293
        & +6.83 & +6.70
        & +4.47 & +3.93 \\
        Exchange
        & 0.372 & 0.413
        & +8.04 & +5.66
        & +3.97 & +2.71 \\
        Solar
        & 0.306 & 0.326
        & +50.40 & +40.08
        & +8.53 & +7.52 \\
        \bottomrule
    \end{tabular}
    }
\end{table}

\begin{figure}[t]
    \centering
    \includegraphics[width=\linewidth]{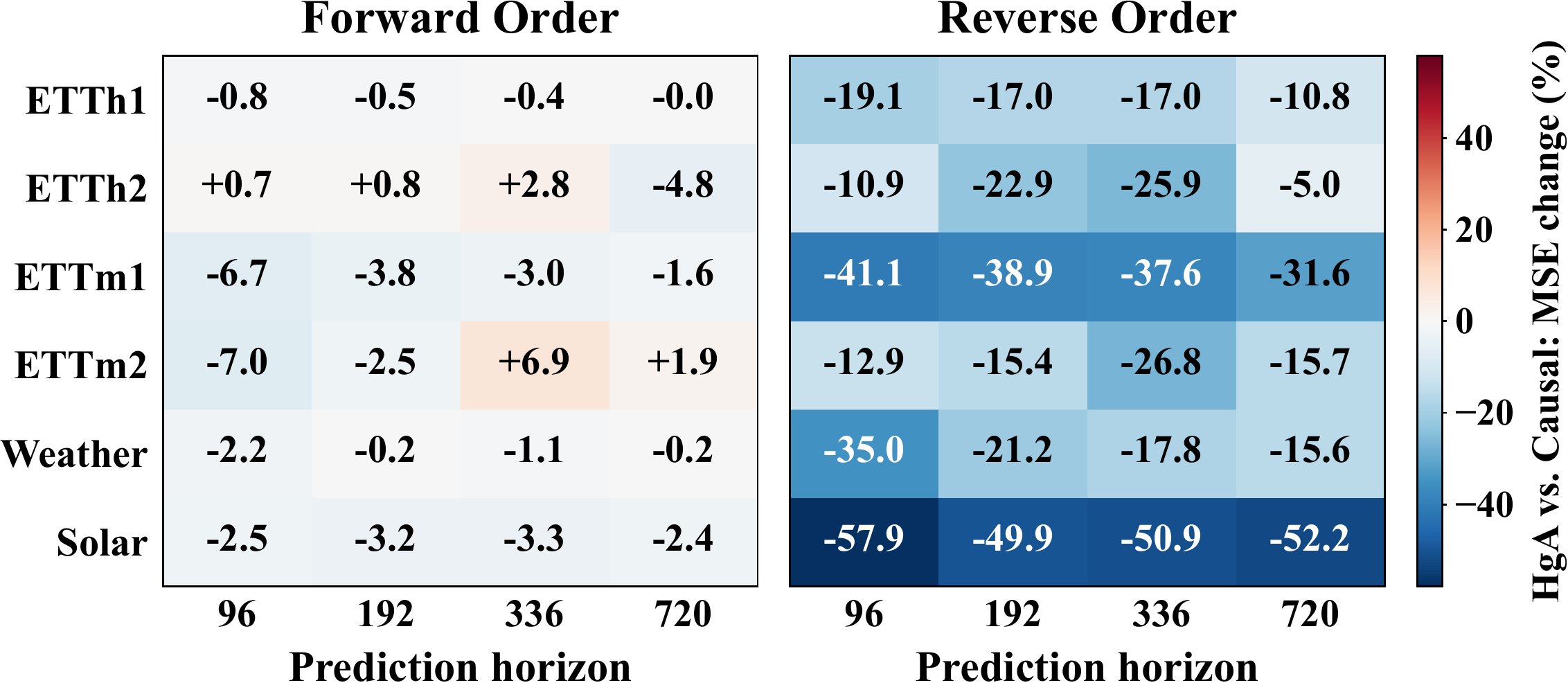}
    \caption{Forecasting Comparison of HgA and Causal Attention under Different Token Orders.}
    \label{fig:hga2_complete_tasks_no_exchange}
\end{figure}

\subsubsection{Predictive Value of Information Pathways}

To evaluate the predictive value of HgA's information pathways, we compare
HgA with causal attention under forward and reversed token orders, using
the same original latest patch as the prediction readout in both settings.
Reversing token order changes the information accessible through causal
attention, allowing us to assess the contribution of HgA's additional
cross-token interactions.
Figure~\ref{fig:hga2_complete_tasks_no_exchange} shows that HgA
outperforms causal attention on most forward-order tasks, with more
pronounced improvements under reversed order, where HgA achieves lower
MSE on all 24 tasks shown.
These results highlight the predictive value of HgA's information
pathways, demonstrating its ability to effectively integrate cross-token
information while maintaining rank stability.

\subsubsection{Empirical Validation of Rank Theory}

\begin{figure}[t]
    \centering
    \includegraphics[scale=0.23]{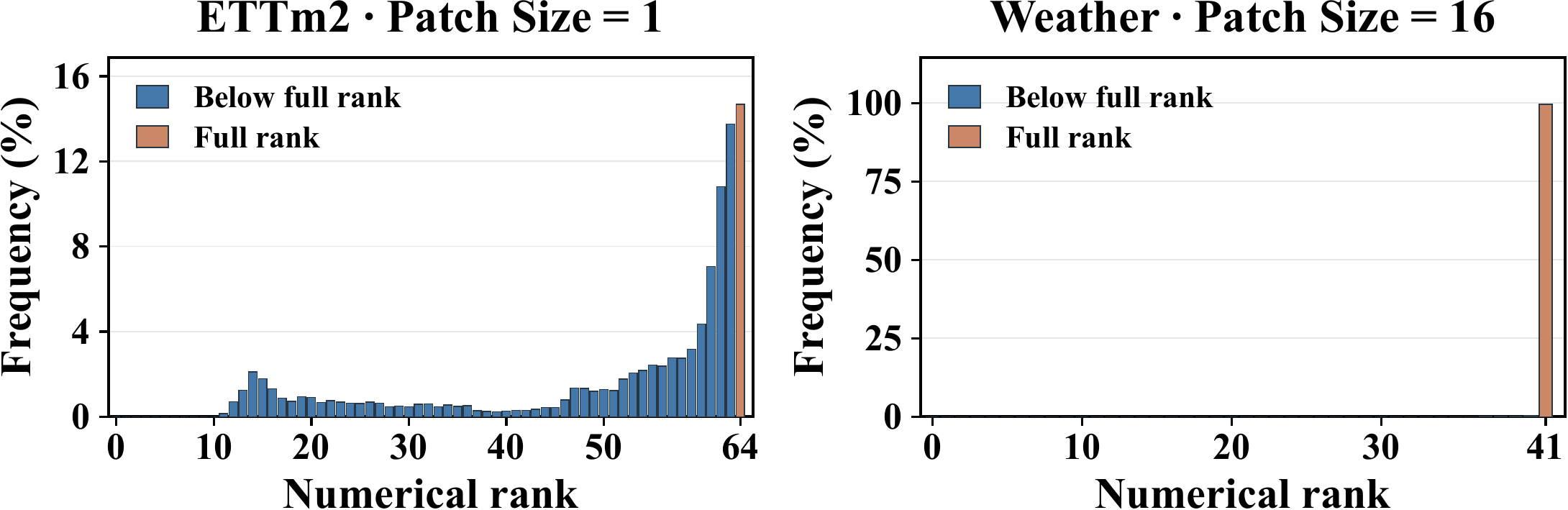} \\ 
    \caption{Rank distributions of embedding matrices with patch sizes $P=1$ (ETTm2) and $P=16$ (Weather).}
    \label{rank_distribution_emb}
\end{figure}

\begin{figure}[t]
    \centering
    \includegraphics[scale=0.225]{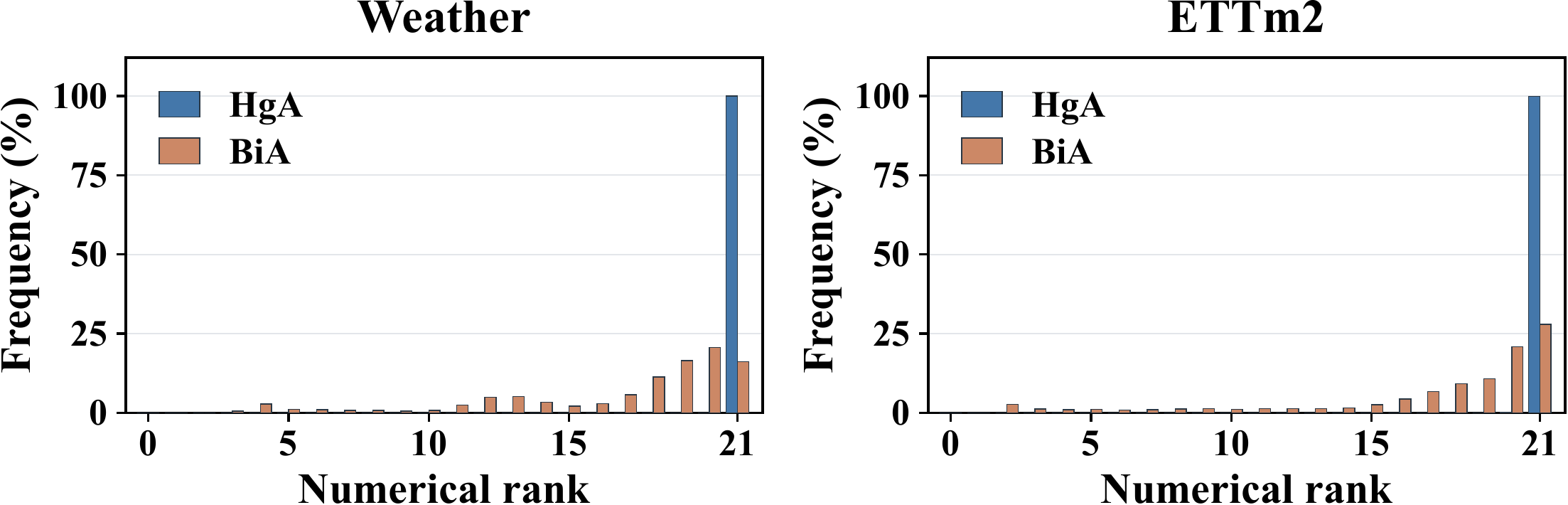}
    \caption{Rank distributions of attention matrices on Weather and ETTm2.}
    \label{rank_distribution_atten}
\end{figure}

To empirically validate our theoretical framework, we visualize the rank distributions of intermediate representations in Figure~\ref{rank_distribution_emb} and ~\ref{rank_distribution_atten}. 
First, the Figure~\ref{rank_distribution_emb} demonstrate that our nonlinear MLP embedding explicitly breaks the rank-1 bottleneck inherent to $P=1$ linear projections. It consistently drives the representations toward high-rank states, closely approaching the theoretical maximum (validating Theorem~\ref{thm:max_rank_embedding}). 
Second, the Figure~\ref{rank_distribution_atten} compare the attention weight matrices. While standard BiA suffers from severe rank collapse, evidenced by a widespread long-tail distribution toward lower ranks, HgA consistently locks at the maximum rank across all test samples. This confirms that HgA effectively preserves representational diversity and circumvents rank-deficient topologies, perfectly aligning with Theorem~\ref{thm:causal_rank_advantage}.

\subsection{Interactive Progressive Decoder Analysis}

\subsubsection{Analysis of Progressive Refinement}
\label{sec:refinement_analysis}

We investigate the effectiveness of progressive refinement by examining
prediction accuracy, error correction, and inference efficiency.
Using fixed checkpoints, we evaluate different refinement steps
$T\in\{1,\ldots,20\}$ across multiple forecasting horizons.
Figure~\ref{fig:T_sensitivity} shows that ETTm1 benefits primarily from
a few refinement steps, whereas Exchange continues to improve as $T$
increases, indicating dataset-dependent sensitivity to refinement depth.

To understand what refinement corrects, we decompose forecasting errors
into level, linear-trend, and remaining-shape components.
Figure~\ref{fig:T_error_components} shows that increasing $T$ from 1 to 5
improves MSE and MAE across all six evaluated datasets at $\tau=96$,
with different dominant correction components across datasets.

Finally, we evaluate the efficiency of offline-folded inference.
The folded predictions match explicit iterative updates within numerical
precision while avoiding online iterations.
As shown in Figure~\ref{fig:folding_latency}, at $T=20$, offline folding
achieves median speedups of $3.77\times$ for the refinement module and
$1.45\times$ for the full model.
These results demonstrate that IPD enables effective progressive
forecast correction with efficient inference.

\begin{figure}[t]
    \centering
    \includegraphics[scale=0.18]{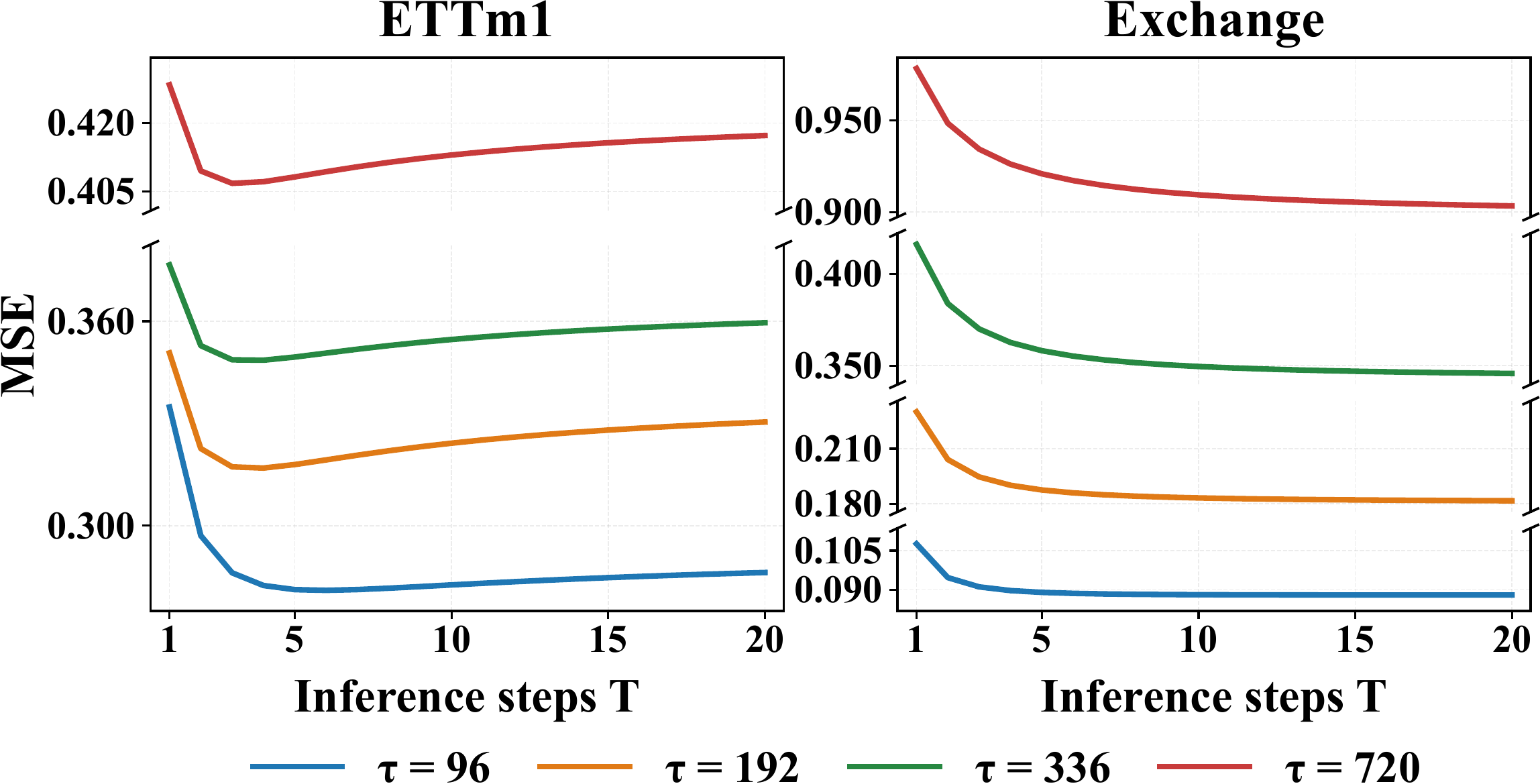}
    \caption{Hyperparameter sensitivity of refinement steps ($T$).}
    \label{fig:T_sensitivity}
\end{figure}

\begin{figure}[t]
    \centering
    \includegraphics[scale=0.17]{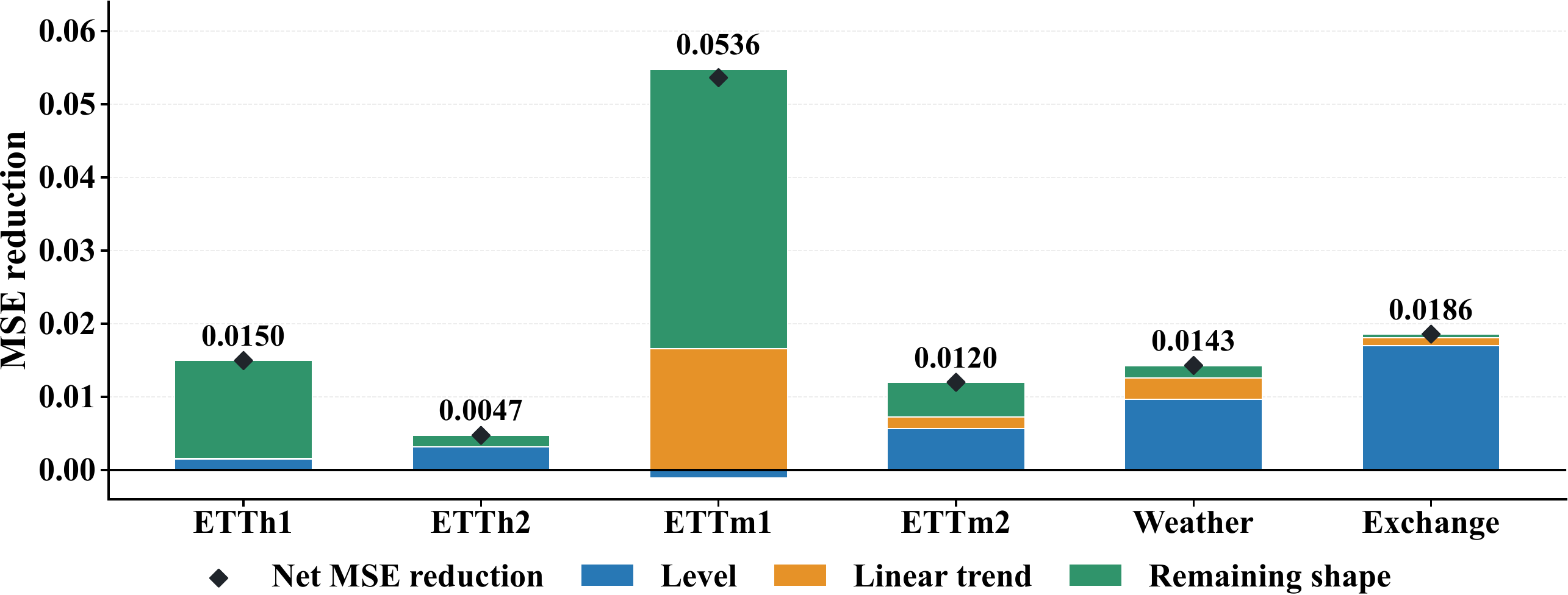}
    \caption{Reductions in level, trend, and remaining-shape errors
from $T=1$ to $T=5$ at $\tau=96$.}
    \label{fig:T_error_components}
\end{figure}

\begin{figure}[t]
    \centering
    \includegraphics[scale=0.2]{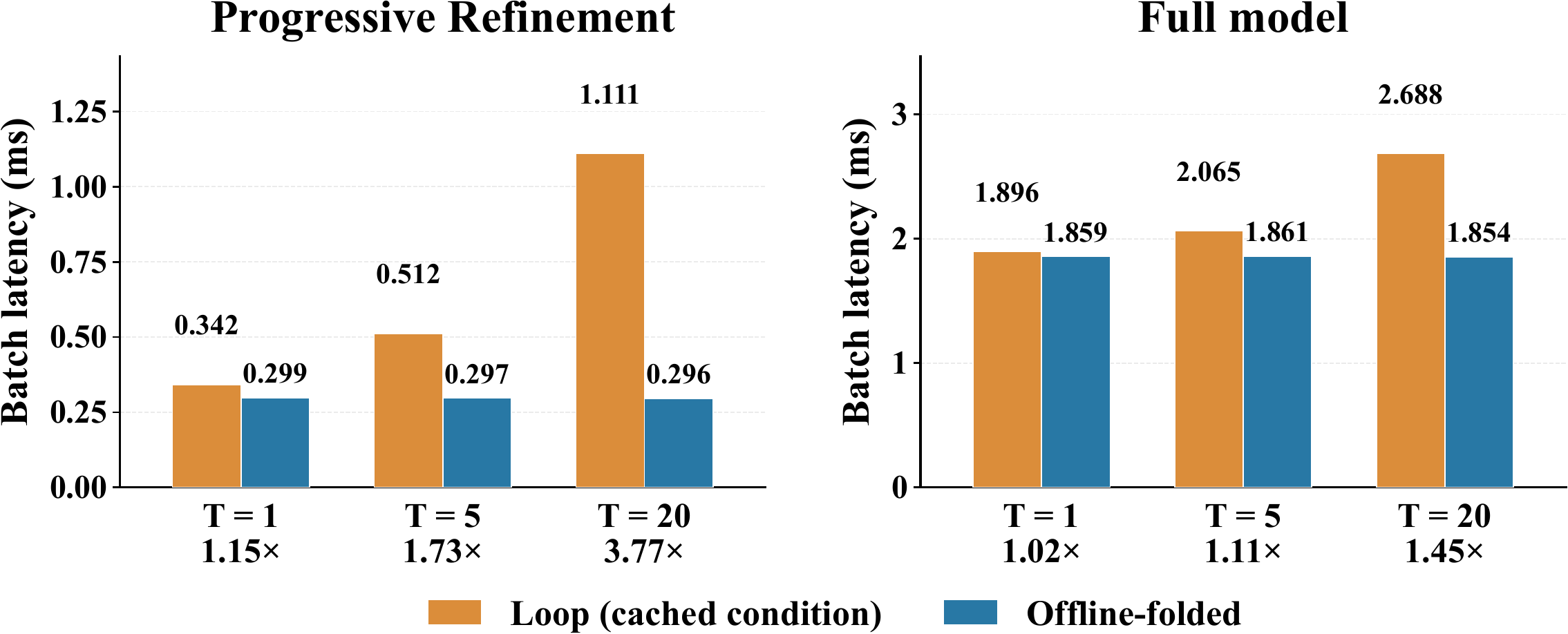}
    \caption{Online inference latency of iterative refinement and
offline-folded inference.}
    \label{fig:folding_latency}
\end{figure}


\subsubsection{Effectiveness of IPD Components.}

To evaluate the contributions of IPD components, we compare three
prediction heads under the same encoder: a standard linear head,
progressive residual refinement without conditioning (Refinement Only),
and the complete IPD with interactive conditioning.
Table~\ref{tab:prediction_head_comparison} shows that residual refinement
alone improves MSE over the linear head on five of six datasets.
Introducing interactive conditioning further improves performance,
with Full IPD achieving the lowest MSE and MAE across all six datasets.
Notably, on Weather, refinement alone slightly underperforms the linear
head, whereas Full IPD improves upon both.
These results demonstrate the complementary benefits of progressive
residual refinement and interactive conditioning for multivariate
forecasting.

\begin{table}[t]
\centering
\caption{Comparison of IPD component configurations.}
\label{tab:prediction_head_comparison}
\begin{tabular}{l|cc|cc|cc}
\toprule
\multirow{2}{*}{Dataset}
& \multicolumn{2}{c|}{Linear Head}
& \multicolumn{2}{c|}{Refinement Only}
& \multicolumn{2}{c}{Full IPD} \\
& MSE & MAE & MSE & MAE & MSE & MAE \\
\midrule
ETTh1
& 0.426 & 0.430
& 0.412 & 0.421
& \textbf{0.405} & \textbf{0.416} \\
ETTh2
& 0.349 & 0.385
& 0.344 & 0.384
& \textbf{0.342} & \textbf{0.384} \\
ETTm1
& 0.349 & 0.368
& 0.342 & 0.367
& \textbf{0.339} & \textbf{0.367} \\
ETTm2
& 0.254 & 0.307
& 0.253 & 0.308
& \textbf{0.252} & \textbf{0.307} \\
Weather
& 0.228 & 0.254
& 0.230 & 0.257
& \textbf{0.226} & \textbf{0.252} \\
Exchange
& 0.405 & 0.426
& 0.395 & 0.421
& \textbf{0.389} & \textbf{0.417} \\
\bottomrule
\end{tabular}
\end{table}

\begin{figure}[t]
    \centering
    \includegraphics[scale=0.14]{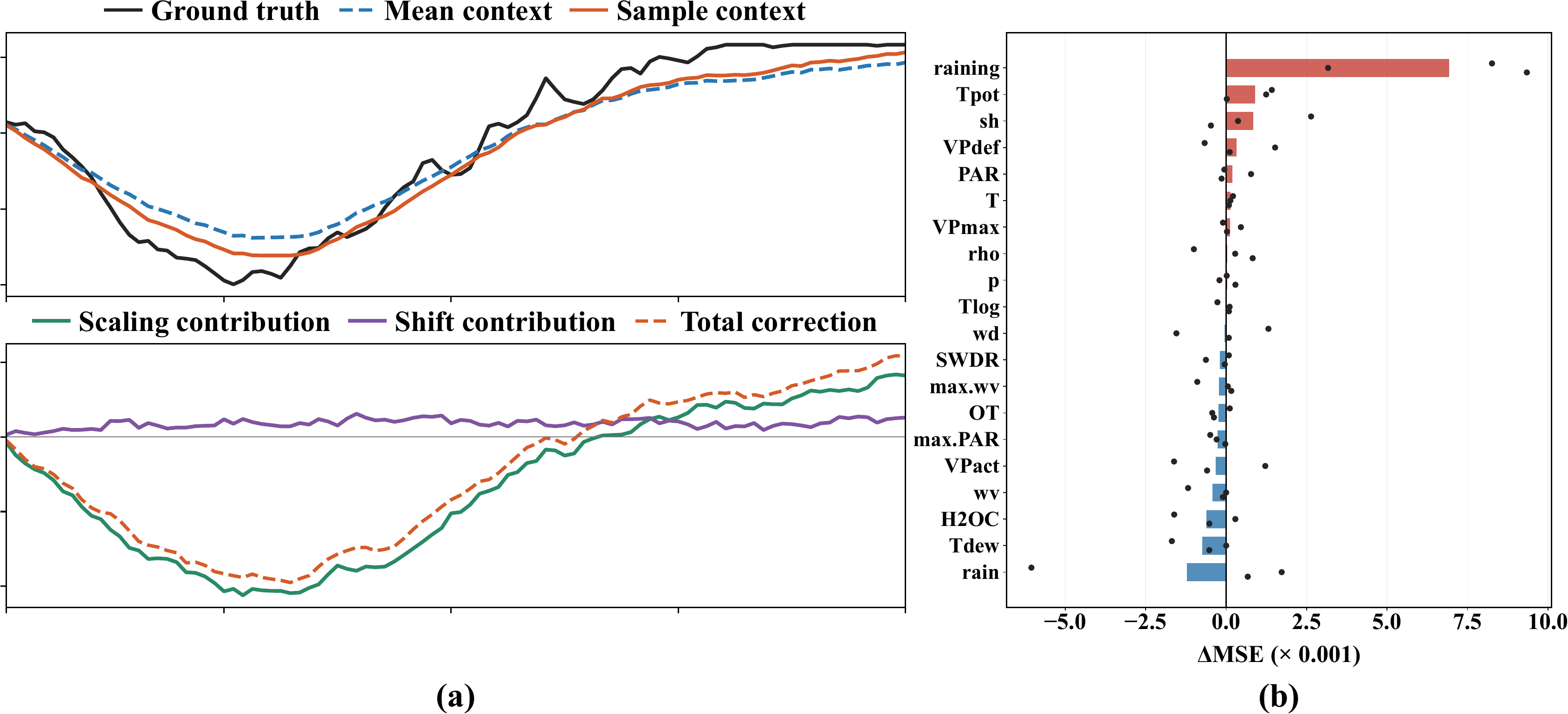}
    \caption{Cross-channel conditioning analysis on Weather.}
    \label{fig:cross_channel_modulation}
\end{figure}

\subsubsection{Effect of Cross-channel Conditioning.}

To investigate how cross-channel information contributes to forecasting,
we analyze the relative-humidity channel on Weather by replacing source
channel features with their training-set means.
Figure~\ref{fig:cross_channel_modulation}(a) illustrates how
sample-dependent cross-channel information adjusts the predicted
trajectory, particularly around the trough and later forecast positions.
The scaling modulation contributes most of the adjustment in this example,
while the shifting modulation provides smaller corrections.
Figure~\ref{fig:cross_channel_modulation}(b) further reveals varying
contributions from different source channels. The \texttt{raining}
channel shows the largest positive contribution.
These findings demonstrate that IPD can exploit informative cross-channel
relationships to produce adaptive temporal corrections, supporting the
effectiveness of interactive conditioning.

\section{Conclusion}
In this paper, we present RankFormer, a principled encoder-decoder
architecture for multivariate time series forecasting. 
Through a rank-based analytical framework, we revisit how tokenization,
embedding, and attention designs influence the capacity of learned
representations. 
Our analysis reveals that patchification expands the attainable rank bound,
nonlinear embeddings overcome the additional constraints of linear
projections, and attention topology plays an important role in preserving
representation diversity during information interaction. 
Guided by these insights, we develop a rank-enhancing encoder with HgA and an
IPD, which respectively improve
representation preservation and future trajectory modeling through effective
information exchange and progressive refinement. 
Extensive experiments on real-world forecasting benchmarks validate the
effectiveness of RankFormer and demonstrate consistent improvements over
competitive methods.

\bibliographystyle{ACM-Reference-Format}
\bibliography{sample-base}

\appendix

\begin{figure}[ht]
    \centering
    \includegraphics[scale=0.6]{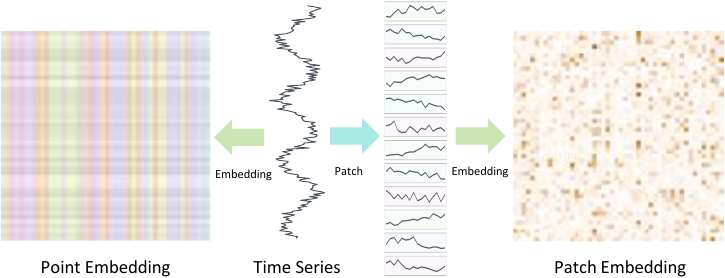}
    \caption{Rank expansion via patchification.}
    \label{fig_fwf}
\end{figure}

\section{Related Work}

\subsection{Time Series Forecasting}
Time series forecasting faces challenges in modeling long-term dependencies and complex dynamics. Traditional RNNs and CNNs suffer from scalability bottlenecks and difficulties in handling long sequences \cite{cheng2020towards, salinas2020deepar}. In contrast, Transformers have achieved superior performance by leveraging self-attention mechanisms, spurring the development of efficient and interpretable temporal models \cite{wutimesnet, liu2022pyraformer} such as Informer \cite{zhou2021informer}, Autoformer \cite{wu2021autoformer}, and PatchTST \cite{nietime}.
Meanwhile, frequency-domain methods effectively capture periodic and spectral patterns in time series through Fourier or wavelet transforms \cite{murad2025wpmixer}. Representative works, including Fredformer \cite{piao2024fredformer}, FreTS \cite{yi2023frequency}, and FITS \cite{xufits}, demonstrate competitive advantages in extracting global features.
Recent studies further integrate temporal and frequency-domain information, proposing hybrid modeling approaches. For example, CDX-Net combines multiple network structures for feature fusion \cite{li2022cdx}, FEDformer unifies decomposition and frequency analysis \cite{zhou2022fedformer}, and TimeMixer++ enhances temporal representation through multi-scale sampling and attention mechanisms, advancing the field \cite{wang2025timemixer++}.
\subsection{Rank Analysis in Transformer}
Low-rank attention research aims to improve Transformer efficiency through low-rank approximations or parameterization. Low-rank approximation methods, such as Linformer \cite{wang2020linformer}, Nystromformer \cite{xiong2021nystromformer}, and Performer \cite{choromanskirethinking}, utilize low-dimensional projections, anchor tokens, or random feature mappings to reduce attention computational complexity to linear. Low-rank parameterization methods \cite{hu2022lora, hayou2024lora+, liu2024dora}, directly learn low-rank update matrices for efficient fine-tuning but lack theoretical guidance on the underlying causes of low-rank behavior.
Theoretical analysis reveals that pure attention networks suffer from a double-exponential collapse in representation rank as depth increases, a phenomenon known as rank collapse \cite{noci2022signal, saadamind}. Studies indicate that residual connections, layer normalization, and feed-forward networks can effectively mitigate this issue and maintain representational diversity \cite{dong2021attention, wu2024role}. Additionally, the relationship between attention head dimension and sequence length influences the rank of attention matrices, thereby constraining model expressiveness \cite{bhojanapalli2020low}.

In time series forecasting, sequence embeddings exhibit rapidly decaying singular value spectra, naturally concentrating in low-rank subspaces and endowing attention layers with inherent compressibility \cite{yu2025understanding}. However, existing work still lacks a unified theoretical framework to systematically guide the overall design and optimization of time series Transformers from a rank perspective.

\section{Visualization of Rank Expansion via Patching}
\label{app:patch_visualization}

Figure~\ref{fig_fwf} illustrates the rank expansion effect discussed in
Section~\ref{sec:rank_theory}. 
Point-wise tokenization ($P=1$) limits the representation to a rank-1 structure, while
patchification enlarges the attainable rank bound to $\min(N,P)$ as stated
in Proposition~\ref{prop:patching_rank}, providing additional independent
directions for subsequent representation learning.

\section{Proof of Theorem~\ref{thm:max_rank_embedding}}
\label{app:proof_thm1}

We consider the two-layer MLP embedding:
\[
\mathbf{H}
=
\phi(\mathbf{X}_p\mathbf{W}_1+\mathbf{1}_N\mathbf{b}_1^\top)
\mathbf{W}_2+\mathbf{1}_N\mathbf{b}_2^\top ,
\]
where $\mathbf{X}_p\in\mathbb{R}^{N\times P}$ denotes the input patch
matrix and $\phi$ is a continuous non-polynomial activation function. We
show that there exist MLP parameters such that
\[
\operatorname{rank}(\mathbf{H})=\min(N,D).
\]

\subsection{Spanning Property of Non-polynomial Activations}

We first establish a property of non-polynomial activations. For any
distinct input points $\{x_i\}_{i=1}^{N}\subset\mathbb{R}^{P}$, define
\[
\Phi(w,b)=
[
\phi(w^\top x_1+b),\ldots,
\phi(w^\top x_N+b)
]^\top .
\]

\begin{lemma}
\label{lem2:activation_spanning}
The set of vectors $\{\Phi(w,b):w\in\mathbb{R}^{P},b\in\mathbb{R}\}$
spans $\mathbb{R}^{N}$.
\end{lemma}

\begin{proof}
Assume that the image of $\Phi$ does not span $\mathbb{R}^{N}$. Then there
exists a non-zero vector $\alpha\in\mathbb{R}^{N}$ satisfying
\[
\sum_{i=1}^{N}\alpha_i\phi(w^\top x_i+b)=0
\]
for all $w$ and $b$.

Since the input points are distinct, there exists a direction
$v\in\mathbb{R}^{P}$ such that the projected values
$\beta_i=v^\top x_i$ are pairwise distinct. Setting $w=tv$ gives
\[
\sum_{i=1}^{N}\alpha_i\phi(t\beta_i+b)=0,
\]
for all $t,b$. This contradicts the discriminatory property of continuous
non-polynomial activations, which states that no non-trivial linear
combination of their translated and scaled functions can vanish
identically. Therefore, the vectors $\Phi(w,b)$ span $\mathbb{R}^{N}$.
\end{proof}

\subsection{Maximum Rank Achievability}

Let
\[
K=\min(N,D).
\]
By Lemma~\ref{lem2:activation_spanning}, the activation outputs span
$\mathbb{R}^{N}$, and therefore there exist $K$ linearly independent vectors
\[
v_1,\ldots,v_K
\]
generated by suitable parameters $(w_j,b_j)$, where
\[
v_j=
[
\phi(w_j^\top x_1+b_j),\ldots,
\phi(w_j^\top x_N+b_j)
]^\top .
\]

Construct the first MLP layer using
\[
\mathbf{W}_1=[w_1,\ldots,w_K],
\qquad
\mathbf{b}_1=[b_1,\ldots,b_K]^\top .
\]

The hidden representation becomes
\[
\mathbf{H}_{hid}
=
\phi(\mathbf{X}_p\mathbf{W}_1+
\mathbf{1}_N\mathbf{b}_1^\top)
=
[v_1,\ldots,v_K],
\]
which satisfies
\[
\operatorname{rank}(\mathbf{H}_{hid})=K.
\]

Finally, choosing the output projection $\mathbf{W}_2$ to map these
independent hidden directions into the output space yields
\[
\operatorname{rank}(\mathbf{H})
=
K
=
\min(N,D).
\]

Therefore, nonlinear MLP embeddings can achieve the maximum attainable rank
under the given sample and hidden dimensions.

\section{Attention Rank Analysis}
\label{app:attention_rank_analysis}

We analyze how attention topology affects the rank preservation of value
representations during token interaction. Given query, key, and value projections:
\[
\mathbf{Q}=\mathbf{X}\mathbf{W}_Q,\quad
\mathbf{K}=\mathbf{X}\mathbf{W}_K,\quad
\mathbf{V}=\mathbf{X}\mathbf{W}_V ,
\]
the attention output after output projection can be written as
\[
\mathbf{O}
=
\mathbf{A}\mathbf{X}_V,
\]
where
\[
\mathbf{X}_V=\mathbf{V}\mathbf{W}^{O},
\] and $\mathbf{A}$ denotes the corresponding attention aggregation matrix.

The following analysis focuses on whether the attention aggregation matrix
$\mathbf{A}$ preserves the rank of the value representation.

\subsection{Full-rank Potential of Bidirectional Attention}

BiA is not inherently rank deficient. We show that its
attention matrix can achieve full rank under suitable projection parameters.

Given an input matrix $\mathbf{X}$ with distinct rows, there exists a
projection direction $\mathbf{v}$ such that
\[
\mathbf{u}=\mathbf{X}\mathbf{v}
\]
contains pairwise distinct values. Choosing one-dimensional query and key
projections as
\[
\mathbf{Q}=\alpha\mathbf{u},\qquad
\mathbf{K}=\beta\mathbf{u},
\]
produces the attention score matrix
\[
\mathbf{S}=\gamma\mathbf{u}\mathbf{u}^{\top}.
\]

Although $\mathbf{S}$ is rank one, the element-wise exponential matrix
generated by the softmax operation,
\[
\mathbf{E}_{ij}=\exp(\gamma u_i u_j),
\]
is full rank for distinct $\{u_i\}$. Since the softmax normalization only
introduces an invertible diagonal scaling, the resulting bidirectional
attention matrix can achieve
\[
\operatorname{rank}(\mathbf{A}_{\textup{bid}})=n.
\]

Therefore, BiA itself does not suffer from an intrinsic
rank limitation. Its limitation is the absence of a structural guarantee for
preserving rank during repeated token interaction.

\subsection{Rank Preservation of Causal Attention}
\label{app:causal_rank}

\begin{lemma}[Invertibility of Causal Attention]
\label{lem:causal_invertible}
The causal attention matrix
$\mathbf{A}_{\textup{causal}}$ is invertible:
\[
\operatorname{rank}(\mathbf{A}_{\textup{causal}})=n.
\]
\end{lemma}

\begin{proof}
Due to the causal mask, $\mathbf{A}_{\textup{causal}}$ is a lower triangular
matrix. Its diagonal entries are:
\[
(\mathbf{A}_{\textup{causal}})_{ii}
=
\frac{
\exp(\mathbf{S}_{ii})
}{
\sum_{j=1}^{i}\exp(\mathbf{S}_{ij})
},
\]
where $\mathbf{S}=\mathbf{Q}\mathbf{K}^{\top}/\sqrt{d_k}$.

Since all exponential terms are strictly positive,
\[
(\mathbf{A}_{\textup{causal}})_{ii}>0 .
\]

\begin{figure*}[ht]
	\centering
	\includegraphics[scale=0.29
    ]{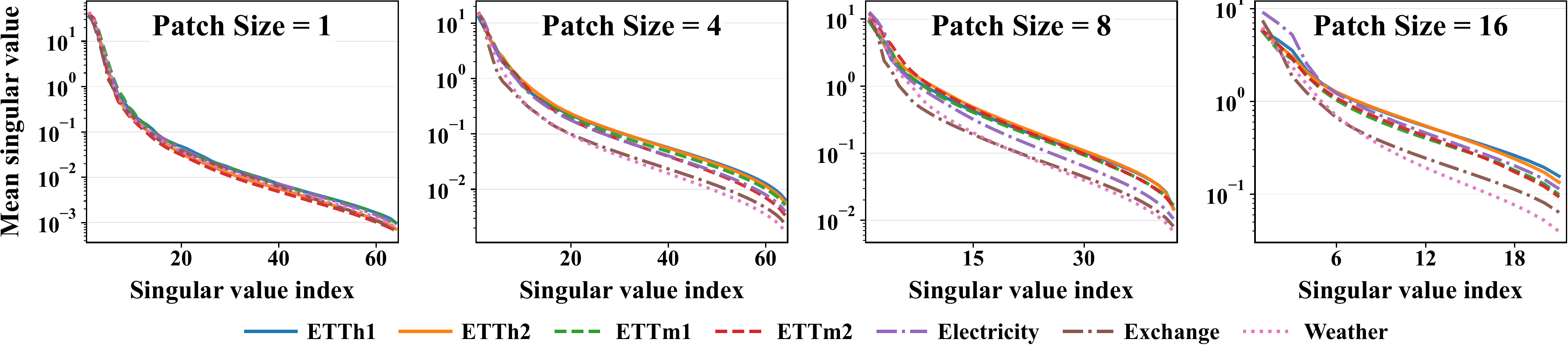}
  \caption{Rank Distribution of Embedding Matrices (Patch Size = 16).}
	\label{patch_comparison}
\end{figure*}

\begin{figure*}[ht]
    \centering
    { \textbf{(a) Bidirectional Attention}} \\
    \includegraphics[width=1\linewidth]{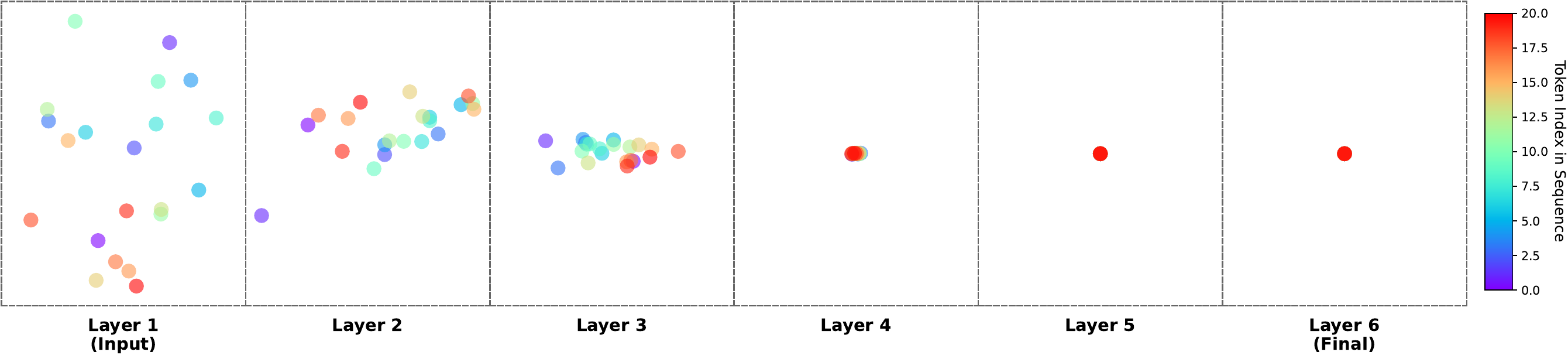} \\
    {\textbf{(b) Hourglass Attention}} \\
    \includegraphics[width=1\linewidth]{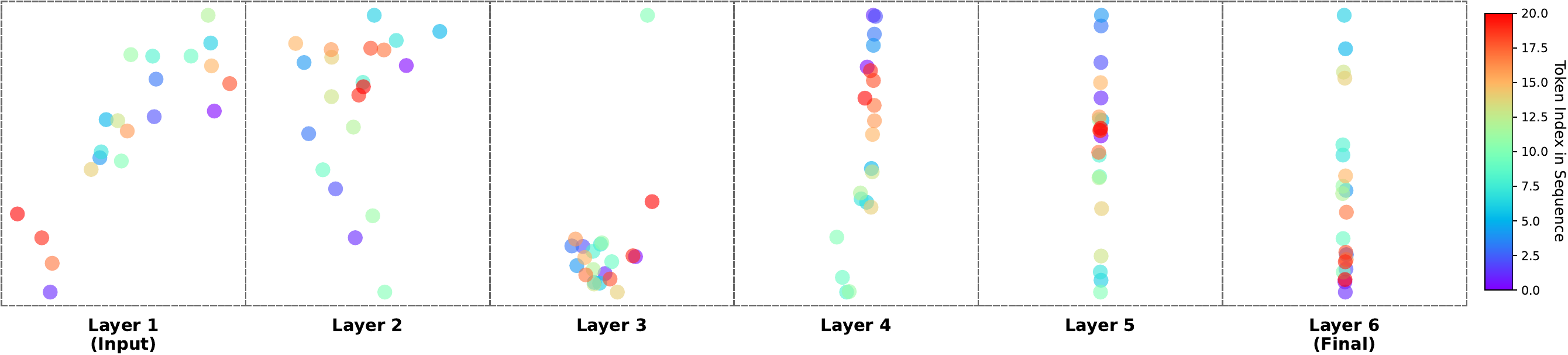}
    \caption{PCA visualization of token representation evolution across deep layers in a pure attention network. The color gradient indicates the token index within the sequence.}
    \label{fig:rep_collapse}
\end{figure*}

Therefore,
\[
\det(\mathbf{A}_{\textup{causal}})
=
\prod_{i=1}^{n}
(\mathbf{A}_{\textup{causal}})_{ii}
>0,
\]
which proves that $\mathbf{A}_{\textup{causal}}$ is invertible.
\end{proof}

\subsubsection{Proof of Theorem~\ref{thm:causal_rank_advantage}}
\label{app:proof_thm2}

\begin{proof}
Let
\[
\mathbf{X}_V=\mathbf{V}\mathbf{W}^{O}.
\]

From Lemma~\ref{lem:causal_invertible},
$\mathbf{A}_{\textup{causal}}$ is invertible. Therefore,
\[
\operatorname{rank}(\mathbf{O}_{\textup{causal}})
=
\operatorname{rank}
(\mathbf{A}_{\textup{causal}}\mathbf{X}_V)
=
\operatorname{rank}(\mathbf{X}_V).
\]

For bidirectional attention,
the rank of
\[
\mathbf{O}_{\textup{bid}}
=
\mathbf{A}_{\textup{bid}}\mathbf{X}_V
\]
is determined by whether the null space of
$\mathbf{A}_{\textup{bid}}$ intersects with the column space of
$\mathbf{X}_V$.

By the rank-nullity theorem restricted to
$\operatorname{Col}(\mathbf{X}_V)$:
\[
\operatorname{rank}(\mathbf{O}_{\textup{bid}})
=
\operatorname{rank}(\mathbf{X}_V)
-
\dim
(
\operatorname{ker}(\mathbf{A}_{\textup{bid}})
\cap
\operatorname{Col}(\mathbf{X}_V)
).
\]

Hence,
\[
\operatorname{rank}(\mathbf{O}_{\textup{causal}})
-
\operatorname{rank}(\mathbf{O}_{\textup{bid}})
=
\dim
(
\operatorname{ker}(\mathbf{A}_{\textup{bid}})
\cap
\operatorname{Col}(\mathbf{X}_V)
)
\geq0.
\]

Therefore,
\[
\operatorname{rank}(\mathbf{O}_{\textup{causal}})
\geq
\operatorname{rank}(\mathbf{O}_{\textup{bid}}).
\]

Equality holds exactly when
\[
\operatorname{ker}(\mathbf{A}_{\textup{bid}})
\cap
\operatorname{Col}(\mathbf{X}_V)
=
\{0\}.
\]
\end{proof}




\section{Additional Experimental Results}

\paragraph{Different Patch Sizes}
As shown in Figure \ref{patch_comparison}, we systematically evaluated the singular value distribution of embedding matrices under different patch sizes (1, 4, 8, 16) across 7 datasets. The experimental results show that for most datasets, the singular value decay curves of their embedding representations generally exhibit full-rank or near-full-rank characteristics, even under the smallest patch size (1). This indicates that the model can learn highly diverse and linearly independent representations for these datasets, with their effective rank significantly surpassing the inherent bottleneck of the original univariate input (which has an upper rank bound of 1).

\paragraph{Visualization of Representation Collapse}

As illustrated in Figure~\ref{fig:rep_collapse}, the evolutionary trajectories of the token representations differ drastically depending on the chosen attention mechanism:
\textbf{BiA:} The top panel (Figure~\ref{fig:rep_collapse}a) demonstrates that while tokens are initially scattered in the input space, they rapidly assimilate during forward propagation. By Layer 4, the representations of all tokens completely collapse into a single indistinguishable point. This visualizes the extreme over-smoothing and severe rank collapse inherent to deep unmasked attention.
\textbf{HgA:} The bottom panel (Figure~\ref{fig:rep_collapse}b) shows that HgA successfully preserves representational diversity. Even at the final layer (Layer 6), the tokens maintain a distinct and structured spread. This provides direct and compelling visual evidence of HgA's structural superiority in circumventing rank collapse and preserving full-rank capacity.

\end{document}